%% file: main.tex
\PassOptionsToPackage{dvipsnames,table}{xcolor}
\documentclass[sigconf]{acmart}

\AtBeginDocument{%
  }

\usepackage{graphicx}
\usepackage{amsmath}
\usepackage{booktabs}
\usepackage{balance}

\usepackage[capitalize]{cleveref}
\crefname{section}{Sec.}{Secs.}
\crefname{table}{Tab.}{Tabs.}
\crefname{equation}{Eq.}{Eq.}
\Crefname{section}{Section}{Sections}
\Crefname{table}{Table}{Tables}
\Crefname{assumption}{Assumption}{Assumptions}

\usepackage{soul}
\usepackage{url}
\usepackage{amsthm}
\usepackage{pifont}
\usepackage{algorithm}
\usepackage{algpseudocode}
\usepackage{caption}
\usepackage{subcaption}
\usepackage{amsfonts}
\usepackage{tabularx}
\usepackage{enumitem}
\usepackage{multirow}

\definecolor{LightBlue}{rgb}{0.86,0.90,0.95}
\definecolor{LightGreen}{rgb}{0.86,0.95,0.90}

\makeatletter
\def\thm@space@setup{\thm@preskip=0pt
\thm@postskip=0pt}
\makeatother

\newtheorem{definition}{Definition}

\newtheorem{theorem}{Theorem}
\newtheorem{assumption}{Assumption}

\newcommand{\sysname}{FairDolce}
\newcommand{\sysnameregret}{FairSDR}

\makeatletter
\newcommand{\multiline}[1]{%
  \begin{tabularx}{\dimexpr\linewidth-\ALG@thistlm}[t]{@{}X@{}}
    #1
  \end{tabularx}
}
\makeatother

\copyrightyear{2023}
\acmYear{2023}
\setcopyright{acmlicensed}
\acmConference[KDD '23]{Proceedings of the 29th ACM SIGKDD Conference on Knowledge Discovery and Data Mining}{August 6--10, 2023}{Long Beach, CA, USA}
\acmBooktitle{Proceedings of the 29th ACM SIGKDD Conference on Knowledge Discovery and Data Mining (KDD '23), August 6--10, 2023, Long Beach, CA, USA}
\acmPrice{15.00}
\acmDOI{10.1145/3580305.3599523}
\acmISBN{979-8-4007-0103-0/23/08}

\begin{document}

\title{Towards Fair Disentangled Online Learning for Changing Environments}

\author{Chen Zhao}
\authornote{Both authors contributed equally to this research.}
\email{chen\_zhao@baylor.edu}
\affiliation{%
  \institution{Baylor University}
  \city{Waco, Texas}
  \country{USA}
}

\author{Feng Mi}
\authornotemark[1]
\email{feng.mi@utdallas.edu}
\affiliation{%
  \institution{University of Texas at Dallas}
  \city{Richardson, Texas}
  \country{USA}}

\author{Xintao Wu}
\email{xintaowu@uark.edu}
\affiliation{%
  \institution{University of Arkansas}
  \city{Fayetteville, Arkansas}
  \country{USA}}

\author{Kai Jiang}
\email{kai.jiang@utdallas.edu}
\affiliation{%
  \institution{University of Texas at Dallas}
  \city{Richardson, Texas}
  \country{USA}}

\author{Latifur Khan}
\email{lkhan@utdallas.edu}
\affiliation{%
  \institution{University of Texas at Dallas}
  \city{Richardson, Texas}
  \country{USA}}

\author{Christan Grant}
\email{christan@ufl.edu}
\affiliation{%
  \institution{University of Florida}
  \city{Gainesville, Florida}
  \country{USA}}
  
\author{Feng Chen}
\email{feng.chen@utdallas.edu}
\affiliation{%
  \institution{University of Texas at Dallas}
  \city{Richardson, Texas}
  \country{USA}}

\renewcommand{\shortauthors}{Chen Zhao et al.}

\begin{abstract}
  \input{abstract.tex}
\end{abstract}




\maketitle

\section{Introduction}
\label{sec:intro}

\input{introduction.tex}

\section{Related Work}
\label{sec:relatedwork}
    \input{relatedwork.tex}

\section{Preliminaries}
\label{sec:preliminaries}
    \input{preliminaries.tex}

\section{Methodology}
\label{sec:method}
    \input{method.tex}

\section{Analysis}
\label{sec:analysis}
    \input{analysis.tex}

\section{Experimental Settings}
\label{sec:experiments}
    \input{experiments.tex}

\section{Results}
\label{sec:results}
    \input{results.tex}

\section{Conclusion}
\label{sec:conclusion}
    \input{conclusion.tex}

\begin{acks}
\begin{sloppypar}
The research reported was supported by the National Science Foundation under grant number 2147375 and 1750911. 
\end{sloppypar}
\end{acks}

\newpage
\newpage
\bibliographystyle{ACM-Reference-Format}
\balance
\bibliography{reference}

\newpage
\newpage
\appendix

\section{Additional Experiment Details}
\label{app:exp-details}
    \input{app_expdetails.tex}

\section{Additional Experiment Results}
\label{app:exp-results}
    \input{app_expresults.tex}

\section{Sketch Proof of Theorem 1}
\label{app:proof}
    \input{app_proof.tex}

\end{document}

%% file: abstract.tex
In the problem of online learning for changing environments, data are sequentially received one after another over time, and their distribution assumptions may vary frequently. Although existing methods demonstrate the effectiveness of their learning algorithms by providing a tight bound on either dynamic regret or adaptive regret, most of them completely ignore learning with model fairness, defined as the statistical parity across different sub-population (\textit{e.g.,} race and gender). Another drawback is that when adapting to a new environment, an online learner needs to update model parameters with a global change, which is costly and inefficient. Inspired by the sparse mechanism shift hypothesis \cite{scholkopf2021toward}, we claim that changing environments in online learning can be attributed to partial changes in learned parameters that are specific to environments and the rest remain invariant to changing environments. To this end, in this paper, we propose a novel algorithm under the assumption that data collected at each time can be disentangled with two representations, an environment-invariant semantic factor and an environment-specific variation factor. The semantic factor is further used for fair prediction under a group fairness constraint. To evaluate the sequence of model parameters generated by the learner, a novel regret is proposed in which it takes a mixed form of dynamic and static regret metrics followed by a fairness-aware long-term constraint. The detailed analysis provides theoretical guarantees for loss regret and violation of cumulative fairness constraints. Empirical evaluations on real-world datasets demonstrate our proposed method sequentially outperforms baseline methods in model accuracy and fairness.


%% file: introduction.tex
Unlike offline learning approaches, where data is accumulated over time and collected at once, online learning assumes data batches are acquired as a continuous flow and sequentially received one after another, making it ideal for the real world. Although online learners can learn from new information in real-time as it arrives, state-of-the-art online learning algorithms may fail catastrophically when learning environments are dynamic and change over time, where changing environments refer to shifted distributions of data features between batches. Therefore, it requires online learning algorithms to adapt dynamically to new patterns in data sequences. 

To address changing environments, adaptive regret \cite{Daniely-2015-ICML} and dynamic regret \cite{Zinkevich-ICML-2003} are introduced. 
Adaptive regret evaluates the learner's performance on any contiguous time intervals, and it is defined as the maximum static regret \cite{Zinkevich-ICML-2003} over these intervals \cite{Daniely-2015-ICML}. 
In contrast, dynamic regret handles changing environments from the perspective of the entire learning process. It allows the comparator changes over time. 
However, minimizing dynamic regret may be less efficient because the learner needs to update model parameters with a global change against changing environments. Inspired by the sparse mechanism shift hypothesis \cite{scholkopf2021toward}, we state changing environments in online learning can be attributed to partial changes of parameters in a long run that are specific to environments. This implies that some parameters remain semantically invariant across different environments. 

Existing fairness-aware online algorithms are developed with a focus on either static or adaptive regret. Learning fairness with dynamic regret for changing environments is barely touched. Data containing bias on some sensitive characters (\textit{e.g.} race and gender) are likely collected sequentially over time. Group fairness is defined by the equality of a predictive utility across different data sub-populations, and predictions of a model are statistically independent on sensitive information. To control bias sequentially, the summation of fair constraints over time added to static loss regret is minimized \cite{OGDLC-2012-JMLR}. It ensures the total violation of fair constraints sublinearly increases in time. Although the adaptive fair regret proposed in \cite{zhao-KDD-2022} is initially designed for online changing environments, it allows the learner to make decisions at some time that do not belong to the fair domain and assumes the total number of times is known in advance. Therefore, designing fairness-aware online algorithms associated with dynamic regret for changing environments becomes desirable.


In this paper, to address the problem of fairness-aware online learning, where a sequence of data batches (\textit{e.g.} tasks) are collected one after another over time with changing task environments (see \cref{fig:RLN-PLN}), we propose a novel regret metric, namely \sysnameregret{}, followed by long-term fairness-aware constraints. To adapt to dynamic environments, we state that shifts in data distributions can be attributed to partial updates in model parameters in a long run, with some remaining invariant to changing environments. Inspired by dynamic and static regret metrics, \sysnameregret{} and the violation of cumulative fair constraints are minimized and bounded with $O(\sqrt{T(1+P_T)})$ and $O(\sqrt{T})$, respectively, where $T$ is the number of iterations and $P_T$ is the path-length of the comparator sequence. To learn a sequence of model parameters satisfying the regret, we propose a novel online learning algorithm, namely \sysname{}. 
In this algorithm, two learning networks are introduced, the representation learning network (RLN) and the prediction learning network (PLN). RLN disentangles an input with environment-invariant and environment-specific representations. It aims to ensure the semantic invariance of the learned presentation from RLN to all possible environments. Furthermore, the environment-invariant representations are used to predict class labels constrained with controllable fair notions in PLN. The main contributions of this paper are summarized\footnote{Code repository: \url{https://github.com/harderbetter/fairdolce}}:

\begin{itemize}[leftmargin=*]
    \item We propose a novel regret \sysnameregret{} that compares the cumulative loss of the learner against any sequence of comparators for changing environments, where only partial parameters need to be adapted  to the changed environments in the long run. The proposed new regret takes a mixed form of static and dynamic regret metrics, subject to a long-term fairness constraint. 
    \item To adapt to changing environments, we postulate that model parameters are updated with a local change. An effective algorithm \sysname{} is introduced, consisting of two networks: a representation learning network (RLN) and a prediction learning network (PLN). In RLN, datapoints are disentangled into two representations. With semantic representations, PLN is optimized under fair constraints.
    \item Theoretically grounded analysis justifies the effectiveness of the proposed method by demonstrating upper bounds $O(\sqrt{T(1+P_T)})$ for loss regret and $O(\sqrt{T})$ for violation of cumulative fair constraints.
    \item We validate the performance of our approach with state-of-the-art techniques on  real-world datasets. Our results demonstrate \sysname{} can effectively adapt both accuracy and fairness in changing environments and it shows substantial improvements over the best prior works.
\end{itemize}



%% file: relatedwork.tex
\textbf{Fairness-aware online learning.}
To sequentially ensure fairness guarantees at each time, a fairness-aware regret \cite{Vishakha-2020-AAAI} considering the trade-off between model accuracy and fairness is devised and it provides a fairness guarantee held uniformly over time. Another trend \cite{zhao-KDD-2021,zhao-KDD-2022,OGDLC-2012-JMLR,AdpOLC-2016-ICML,GenOLC-2018-NeurIPS} addressing this problem is to develop a new metric by adding a long-term fair constraint directly to the loss regret. However, when handling constrained optimization problems, the computational burden of the projection onto the fair domain may be too high when constraints are complex. For this reason, \cite{OGDLC-2012-JMLR} relaxes the output through a simpler closed-form projection. Thereafter, a number of variants of \cite{OGDLC-2012-JMLR} are proposed with theoretical guarantees by modifying stepsizes in \cite{OGDLC-2012-JMLR} to an adaptive version, adjusting to stochastic constraints \cite{Yu-2017-NIPS}, and clipping constraints into a non-negative orthant \cite{GenOLC-2018-NeurIPS}. Although such techniques achieve state-of-the-art bounds for static regrets and violation of fair constraints, they assume datapoints sampled at each time from a stationary distribution and make heavy use of the \textit{i.i.d} assumption. This does not hold when the environment changes.

\textbf{Online learning for changing environments.}
Because low static regret does not imply a good performance in changing environments, two regret metrics, dynamic regret \cite{Zinkevich-ICML-2003} and adaptive regret \cite{Hazan-2007-ARegret}, are devised to measure the learner's performance in changing environments. Adaptive regret handles changing environments from a local perspective by focusing on comparators in short intervals, in which geometric covering intervals \cite{Daniely-2015-ICML,Jun-2017-AISTATS,zhang-2020-AISTATS} and data streaming techniques \cite{gyorgy-2012-efficient} are developed. CBCE \cite{Jun-2017-AISTATS} improved the strongly adapted regret bound by combing the sleeping bandits idea with the Coin Betting algorithm. AOD \cite{zhang-2020-AISTATS} targets both dynamic and adaptive regret and proposes theoretic guarantees to minimize both regrets simultaneously. Although existing methods achieve state-of-the-art performance, a major drawback is that they immerse in minimizing objective functions but ignore the model fairness of prediction. As the first work addressing the problem of online fairness learning for changing environments, FairSAOML \cite{zhao-KDD-2022} combines tasks with a number of sets with different lengths and develops an effective algorithm inspired by expert-tracking techniques. A major drawback of FairSAOML is that (1) it assumes some tasks are known in advance which leads to delays during the learning process; (2) by designing intervals with long lengths, it is hard for a learner to adapt to new environments without leaving information from past environments behind. As a consequence, the adaptation of the learner to new environments may not perform well.

With concerns from existing works, to tackle the problem of fairness-aware online learning for changing environments, in this paper, we propose a novel regret and a learning algorithm, in which we assume only part of the model parameters is responsible for adapting to new environments and the rest are environment-invariant corresponding to fair predictions. Inspired by invariant learning strategies, the proposed algorithm \sysname{} is used to accommodate changing environments and adaptively learn the model with accuracy and fairness.

%% file: preliminaries.tex
Vectors are denoted by lowercase boldface letters. Scalars are denoted by lowercase italic letters. Sets are denoted by uppercase calligraphic letters. For more details refer to \cref{sec:notations}.

\subsection{Online Learning}
\label{sec:online-learning}
In online learning, data batches $\mathcal{D}_t$, defined as tasks, arrive one after another over time. An online machine learner can learn from new information in real-time as they arrive. Specifically, at each time, the learner faces a loss function $f_t:\mathbb{R}^d\times\Theta\rightarrow\mathbb{R}$ which does not need to be drawn from a fixed distribution and could even be chosen adversarially over time \cite{Finn-ICML-2019}. The goal of the learner over all times $T$ is to decide a sequence of model parameters $\{\boldsymbol{\theta}_t\}_{t=1}^T$ by an online learning algorithm, \textit{e.g.,} follow the leader \cite{Hannan-FTL-1957}, that performs well on the loss sequence $\{f_t(\mathcal{D}_t,\boldsymbol{\theta}_t)\}_{t=1}^T$. Particularly, to evaluate the algorithm, a standard objective for online learning is to minimize some notion of regret, defined as the overall difference between the learner's loss $\sum\nolimits_{t=1}^T f_t(\mathcal{D}_t,\boldsymbol{\theta}_t)$ and the best performance achievable by comparators. 

\textbf{Static regret.}
In general, one assumes that tasks collected over time are sampled from a fixed and stationary environment following the \textit{i.i.d} assumption. Therefore, with a sequence of model parameters learned from the learner, the objective is to minimize the accumulative loss of the learned model to that of the best \textit{fixed} comparator $\boldsymbol{\theta}\in\Theta$ in hindsight. This regret is typically referred to as static regret since the comparator is time-invariant.
\begin{equation}
    \small
\begin{aligned}
\label{eq:static-regret}
    R_s = \sum\nolimits_{t=1}^T f_t(\mathcal{D}_t,\boldsymbol{\theta}_t) - \min_{\boldsymbol{\theta}\in\Theta}\sum\nolimits_{t=1}^T f_t(\mathcal{D}_t,\boldsymbol{\theta})
\end{aligned}
\end{equation}
The goal of online learning under a stationary environment is to design algorithms such that static regret $R_s$ sublinearly grows in $T$. However, low static regret does not necessarily imply a good performance in changing environment, where tasks are sampled from various distributions, since the time-invariant comparator $\boldsymbol{\theta}$ in \cref{eq:static-regret} may behave badly. Tasks sequentially collected from non-stationary environments and distributions of them varying over time are more realistic. To address this limitation, recent advances \cite{zhang2017improved,yang-2016-icml} have introduced enhanced regret metrics, \textit{i.e.,} dynamic regret, to measure the learner's performance.

\textbf{Dynamic regret.} The dynamic regret \cite{Zinkevich-ICML-2003} is defined as the difference between the cumulative loss of the learner and that of \textit{a sequence} of comparators $\mathbf{u}_1,\cdots,\mathbf{u}_T\in\Theta$.
\begin{equation}
  \small  
\begin{aligned}
\label{eq:dynamic-regret}
    R_d = \sum\nolimits_{t=1}^T f_t(\mathcal{D}_t,\boldsymbol{\theta}_t) - \sum\nolimits_{t=1}^T f_t(\mathcal{D}_t,\mathbf{u}_t)
\end{aligned}
\end{equation}
In fact, \cref{eq:dynamic-regret} is more general since it holds for any sequence of comparators and thus includes the static regret in \cref{eq:static-regret}. Therefore, minimizing dynamic regret can automatically adapt to the nature of environments, either stationary or dynamic. However, distinct from static regret, bounding dynamic regret is challenging because one needs to establish a universal guarantee that holds for any sequence of comparators \cite{zhang-nips-2018}. An alternative solution for this challenge is to bound the regret in terms of some regularities of the comparator sequence, \textit{e.g.,} path-length \cite{Zinkevich-ICML-2003} defined in \cref{eq:path-length} which measures the temporal variability of the comparator sequence. 

As alluded to in \cref{sec:intro}, most of the state-of-the-art online techniques ignore the significance of learning by being aware of model fairness, which is an important hallmark of human intelligence. To control bias, especially ensure group fairness across different sub-populations, cumulative fairness notions are considered as constraints added on regrets.

\subsection{Group Fairness}
\label{sec:group-fairness}
In general, group fairness criteria used for evaluating and designing machine learning models focus on the relationships between the sensitive variables and the system output \cite{Zhao-ICDM-2019,Wu-2019-WWW,Zhao-ICKG-1-2020}. The problem of group unfairness prevention can be seen as a constrained optimization problem. For simplicity, we consider one binary sensitive label, \textit{e.g.} gender, in this work. However, our ideas can be easily extended to many sensitive labels with multiple levels.

Let $\mathcal{P=X\times Z\times Y\times E}$ be the data space, where $\mathcal{X}\in\mathbb{R}^d$ is an input feature space, $\mathcal{Z}\in\{-1,1\}$ is a sensitive space, $\mathcal{Y}\in\{0,1\}$ is an output space for binary classification, and $\mathcal{E}\in\mathbb{N}$ denotes an environment space. Given a task $\mathcal{D}=\{(\mathbf{x}_{i}, z_{i}, y_{i}, e_i)\}_{i=1}^n \in\mathcal{P}$ in environment $e_i\in\mathcal{E}$, a fine-grained measurement to ensure group fairness in class label prediction is to design fair classifiers by controlling the notions of fairness between sensitive subgroups $\{z_i=1\}_{i=1}^{n_1}$ and $\{z_i=-1\}_{i=1}^{n_{-1}}$ where $n_1+n_{-1}=n$, \textit{e.g.,} demographic parity \cite{Wu-2019-WWW,Lohaus-2020-ICML}.
\begin{definition}[Notions of Fairness \cite{Wu-2019-WWW,Lohaus-2020-ICML,zhao-KDD-2022}]
A classifier $\omega:\mathbb{R}^d\times\Theta\rightarrow\mathbb{R}$ is fair when its predictions are independent of the sensitive attribute $\mathbf{z}=\{z_i\}_{i=1}^n$.
To get rid of the indicator function and relax the exact values, a linear approximated form of the difference between sensitive subgroups is defined \cite{Lohaus-2020-ICML},
\begin{equation}
\small
\begin{aligned}
\label{def:fairness-notion}
    g(\mathcal{D},\boldsymbol{\theta})=\mathbb{E}_{(\mathbf{x},z,y,e)\sim\mathcal{P}}\Big[\frac{1}{\hat{p}_1(1-\hat{p}_1)}\Big(\frac{z+1}{2}-\hat{p}_1\Big)\omega(\mathbf{x}, \boldsymbol{\theta})\Big] 
\end{aligned}
\end{equation}
where $\hat{p}_1$ is an empirical estimate of $pr_1$. $pr_1$ is the proportion of samples in group $z=1$ and correspondingly $1-pr_1$ is the proportion of samples in group $z=-1$. 
\end{definition}
Notice that, in \cref{def:fairness-notion}, when $\hat{p}_1=\mathbb{P}_{(\mathbf{x},z,y,e)\in\mathcal{P}}(z=1)$, the fairness notion $g(\mathcal{D},\boldsymbol{\theta})$ is defined as the difference of demographic parity (DDP). Similarly, when $\hat{p}_1=\mathbb{P}_{(\mathbf{x},z,y,e)\in\mathcal{P}}(y=1, z=1)$, $g(\mathcal{D},\boldsymbol{\theta})$ is defined as the difference of equality of opportunity (DEO) \cite{Lohaus-2020-ICML}.
Therefore, parameters $\boldsymbol{\theta}$ in the domain of a task are feasible if they strictly satisfy the fairness constraint $g(\mathcal{D},\boldsymbol{\theta})=0$. 

\textbf{Motivations.} To tackle the problem of fairness-aware online learning in changing environments, a learner needs to update model parameters with a global change, which is costly and inefficient. Inspired by the sparse mechanism shift hypothesis \cite{scholkopf2021toward}, we state changing environments in online learning can be attributed to partial changes in learned parameters in the long run that are specific to environments. This implies that some parameters remain semantically invariant across different environments.

%% file: method.tex
\subsection{Settings and Problem Formulation}
\label{sec:settings and problem formulation}
We consider a general sequential setting where a learner is faced with tasks $\{\mathcal{D}_t\}_{t=1}^T$ one after another. Each of these tasks corresponds to a time, denoted as $t\in[T]$. At each time, the goal of the learner is to determine model parameters $\boldsymbol{\theta}_t$ using existing task pool $\{\mathcal{D}_i\}_{i=1}^{t-1}$ in a fair domain $\Theta$ that perform well for the task arrived at $t$.  
This is monitored by the loss function $f_t$ and the fairness notion $g_t$, wherein the fair constraint $g_t(\mathcal{D}_{t},\boldsymbol{\theta}_t)=0$ is satisfied and $f_t(\mathcal{D}_{t},\boldsymbol{\theta}_t)$ is minimized. To adapt to changing environments, crucially, model parameters $\boldsymbol{\theta}_t=\{\boldsymbol{\theta}_t^s,\boldsymbol{\theta}_t^v,\boldsymbol{\theta}_t^d,\boldsymbol{\theta}_t^{cls}\}$ can be partitioned into multiple elements, specifically in which
$\boldsymbol{\theta}_t^s$ captures the semantic information of data through a semantic encoder $h_s:\mathcal{X}\times\Theta\rightarrow\mathcal{S}$, and $\boldsymbol{\theta}^{cls}_t$ is used for prediction under fair constraints. $\boldsymbol{\theta}_t^v$ and $\boldsymbol{\theta}_t^d$
are parameters, later introduced in \cref{sec:assumptions,sec:learning-dynamic}, for encoding the environmental information and decoding latent representations, respectively, in order to adaptively train a good $\boldsymbol{\theta}^{s}_t$.
For data batches sampled from heterogeneous distributions at different times, $\boldsymbol{\theta}^s_t$ corresponds to adapting to changing environments by encoding samples to a latent semantic space. With latent factors (representations) encoded from the semantic space as inputs, $\boldsymbol{\theta}^{cls}_t$ is time-invariant in the long run.
\begin{figure*}[!t]
    \centering
    \includegraphics[width=0.9\linewidth]{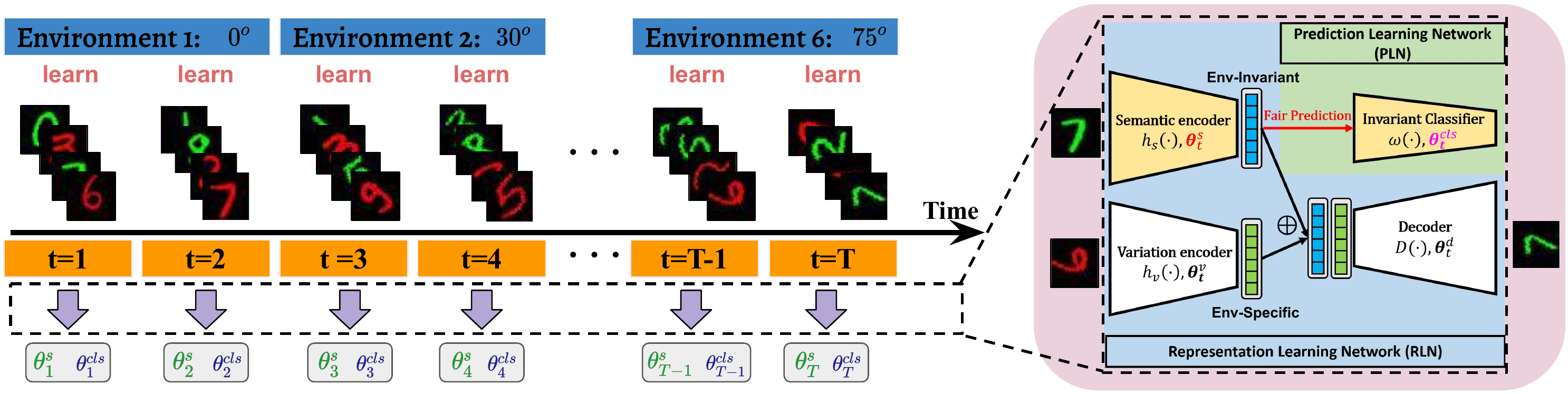}
    \vspace{-3mm}
    \caption{A graphical illustration of the proposed framework using Rotated-Colored-MNIST dataset. \textbf{(Left)} Each angle within $\{0, 15, 30, 45, 60, 75\}$ represents an environment. 
    In the problem of fairness-aware online learning for changing environments, data batches arrive one after another over time. Parameters sequence $\{\boldsymbol{\theta}_t^s,\boldsymbol{\theta}_t^{cls}\}_{t=1}^T$ are learned through the proposed model on the right.
    \textbf{(Right)} The model consists of two learning networks, RLN and PLN. The semantic and variation encoders disentangle an input with two factors (representations). Under \cref{assump:data-inv,assump:class-inv}, the decoder takes both factors and generates new data by diversifying the variation across environments. Semantic factors go through the classifier presented in PLN under fair constraints and further output fair predictions. We claim that when $T$ is large enough, only a subset of the parameters sequence, $\{\boldsymbol{\theta}_t^s\}_{t=1}^T$, are updated to adapt to changing environments.
    }
    \label{fig:RLN-PLN}
    \vspace{-3mm}
\end{figure*}
The overall protocol for this setting is as follows:
\begin{enumerate}[leftmargin=*]
    \item The learner selects semantic parameters $\boldsymbol{\theta}^s_t$ and classification parameters $\boldsymbol{\theta}^{cls}_t$ in the fair domain $\Theta$.
    \item The world reveals a loss and fairness notion $f_t$ and $g_t$.
    \item The learner incurs an instantaneous loss $f_t(h_s(\mathcal{D}_{t},\boldsymbol{\theta}_t^s),\boldsymbol{\theta}^{cls}_t)$ and fairness estimation $g(h_s(\mathcal{D}_{t},\boldsymbol{\theta}_t^s),\boldsymbol{\theta}^{cls}_t)$. 
    \item Advance to the next time.
\end{enumerate}
As mentioned in \cref{sec:online-learning}, the goal of the learner is to minimize regret under long-term constraints \cite{OGDLC-2012-JMLR}, defined as the summation of fair constraints over time. Since $\boldsymbol{\theta}^{s}_t$ adapts to different environments to encode semantic information from a latent invariant space and $\boldsymbol{\theta}^{cls}_t$ further takes semantic inputs for fair prediction, let $\{\boldsymbol{\theta}^{s}_t, \boldsymbol{\theta}^{cls}_t\}_{t=1}^T$ be the sequence of parameters generated at the \textit{Step} (1) of the protocol. We propose \textit{a novel fairness-aware regret for changing environments}, namely \sysnameregret{}, defined as 
\begin{equation}
\small
\begin{aligned}
\label{eq:our_regret}
    &\sysnameregret{} =\sum_{t=1}^T f_t(h_s(\mathcal{D}_{t},\boldsymbol{\theta}_t^s),\boldsymbol{\theta}^{cls}_t) -\min_{\boldsymbol{\theta}^{cls}\in\Theta} \sum_{t=1}^T f_t(h_s(\mathcal{D}_{t},\mathbf{u}_t^s),\boldsymbol{\theta}^{cls})\nonumber\\
    &\text{subject to} \quad \sum_{t=1}^T \Big|\Big|\big[g(h_s(\mathcal{D}_{t},\boldsymbol{\theta}_t^s),\boldsymbol{\theta}^{cls}_t)\big]_+\Big|\Big| = 0
\end{aligned}
\end{equation}
where $[\cdot]_+$ is the projection onto the non-negative space.
Similar to $\{\mathbf{u}_1,\cdots,\mathbf{u}_T\}$ denoted in \cref{eq:dynamic-regret}, $\{\mathbf{u}^s_1,\cdots,\mathbf{u}^s_T\}$ are a sequece of semantic comparators to $\{\boldsymbol{\theta}^s_1,\cdots,\boldsymbol{\theta}^s_T\}$, where each corresponds to an underlying environment. $\boldsymbol{\theta}^{cls}$ is the best-fixed comparator for fair classification, which is time-invariant.

\textbf{Remarks.} In contrast to the regret proposed in \cite{zhao-KDD-2022} in which it is extended from the interval-based strongly adaptive regret, and it aims to minimize the maximum static regret for all intervals on the undivided model parameter, \sysnameregret{} takes the mixed form of static and dynamic regrets. Furthermore, \cite{zhao-KDD-2022} employs the meta-learning framework in which the function inside $f_t$ is designed for interval-level learning with gradient steps on $\boldsymbol{\theta}$. However, in \sysnameregret{}, $h_s$ encodes an input to a semantic representation through a neural network on $\boldsymbol{\theta}^{s}$, which is part of $\boldsymbol{\theta}$.

\subsection{Assumptions for Invariance}
\label{sec:assumptions}
Recall that in the learning protocol mentioned in \cref{sec:settings and problem formulation}, the main goal for the learner is to generate the parameter sequence $\{\boldsymbol{\theta}^{s}_t, \boldsymbol{\theta}^{cls}_t\}_{t=1}^T$ in \cref{sec:settings and problem formulation} that performs well on the loss sequence and the long-term fair constraints. We make the following assumptions.
\begin{assumption}[Shared Semantic Space]
\label{assump:partially-shared-latent-spaces}
\begin{sloppypar}
Given a task $\{(\mathbf{x}_i, z_i, y_i, e_i)\}_{i=1}^{n}$ sampled from a particular environment $e_i\in\mathcal{E}$, we assume that each datapoint in the task is generated from 
\end{sloppypar}
\begin{itemize}[leftmargin=*]
    \item a semantic factor $\mathbf{s}_i=h_s(\mathbf{x}_i,\boldsymbol{\theta}^s)\in\mathcal{S}$, where $\mathcal{S}$ refers to a semantic space shared by all environment $\mathcal{E}$;
    \item a variation factor $\mathbf{v}_i=h_v(\mathbf{x}_i,\boldsymbol{\theta}^v)\in\mathcal{V}$ where $\mathbf{v}_i$ is specific to the individual environment $e_i$.
\end{itemize}
where $h_v:\mathcal{X}\times\Theta\rightarrow \mathcal{V}$ is a variation encoder parameterized by $\boldsymbol{\theta}^v$. We assume that each environment $e_i$ is represented by specific variation factor $h_v(\mathbf{x}_i,\boldsymbol{\theta}^v)$.  
\end{assumption}
This assumption is closely related to the shared latent space assumption in \cite{liu2017unsupervised}, wherein \cite{liu2017unsupervised} assumes a fully shared latent space. We postulate that only the semantic space can be shared across environments whereas the variation factor is environment specific, which is a more reasonable assumption when the cross-environment mapping is many-to-many. In other words, given datapoints in various environments, each can be encoded into semantic and variation factors within the same semantic space but with different variation factors depending on the environments. 

Under \cref{assump:partially-shared-latent-spaces}, each datapoint is able to be disentangled with semantic and variation factors. With two datapoints sampled from the same environment $e_i$, given a decoder $D:\mathcal{S}\times\mathcal{V}\times\Theta\rightarrow\mathcal{X}$, we assume that
\begin{assumption}[Data Invariance under Homogeneous Environments]
\label{assump:data-inv}
    Given a semantic encoder $h_s$, a variation encoder $h_v$, and a decoder $D$, for any $\mathbf{x}_i,\mathbf{x}_j\in\mathcal{X}, i\neq j$ sampled in the same environment $e\in\mathcal{E}$, it holds $\mathbf{x}_i=D(h_s(\mathbf{x}_i,\boldsymbol{\theta}^s), h_v(\mathbf{x}_j,\boldsymbol{\theta}^v), \boldsymbol{\theta}^d)$.
\end{assumption}
\cref{assump:data-inv} enforces the data invariance of the original input $\mathbf{x}_i$ and the one that $D(h_s(\mathbf{x}_i,\boldsymbol{\theta}^s), h_v(\mathbf{x}_j,\boldsymbol{\theta}^v), \boldsymbol{\theta}^d)$ reconstructs jointly from semantic and variation latent factors when the latter remains but the former varies. 

\begin{assumption}[Class Invariance under Heterogeneous Environments \cite{zhang2022towards}]
\label{assump:class-inv}
We assume that inter-environment variation is solely characterized by the environment shift in the distribution $\mathbb{P}(X, E)$. As a consequence, we assume that $\mathbb{P}(Y|X,E)$ is stable across environments. Similar to \cite{zhang2022towards,robey2021model}, given two datapoints $(\mathbf{x}_i, z_i, y, e_i)$ and $(\mathbf{x}_j, z_j, y, e_j)$, we assume the following holds
\begin{equation*}
\small
    \begin{aligned}
        \mathbb{P}(Y=y|X=\mathbf{x}_i,E=e_i) = &\mathbb{P}(Y=y|(X=D(h_s(\mathbf{x}_i,\boldsymbol{\theta}^s), h_v(\mathbf{x}_j,\boldsymbol{\theta}^v),\\
        &\boldsymbol{\theta}^d),E={e_j}), \\
        \forall \mathbf{x}_i,\mathbf{x}_j\in\mathcal{X}, e_i,e_j\in\mathcal{E}, i&\neq j
    \end{aligned}
\end{equation*}
\end{assumption}
This assumption shows that the prediction depends only on the semantic factor $h_s(\mathbf{x},\boldsymbol{\theta}^s)$ regardless of the variation factor $h_v(\mathbf{x},\boldsymbol{\theta}^v)$. Furthermore, the semantic factors are used for fair prediction under fairness constraints.

\subsection{Learning Dynamically for Adaptation}
\label{sec:learning-dynamic}
As the motivation stated in \cref{sec:preliminaries}, an efficient online algorithm is expected to partially update model parameters (\textit{i.e.,} $\boldsymbol{\theta}^s_t$) to adapt to changing environments sequentially and to remain the rest (\textit{i.e.,} $\boldsymbol{\theta}^{cls}_t$). 
As the illustration shown in \cref{fig:RLN-PLN}, a novel online framework for changing environments is proposed with two separate networks. The representation learning network (RLN) aims to learn a good semantic encoder $h_s$ that is able to accurately disentangle semantic representations within various environments, associated with the variation encoder $h_v$ and the decoder $D$. The prediction learning network (PLN) solely consists of the classifier $\omega$ and it takes semantic representations from RLN and outputs fair predictions under fair constraints, which is invariant to environments.

Specifically in RLN, to learn a good semantic encoder $h_s$, at each time $t$ we consider a data batch $\mathcal{Q}_t = \{(\mathbf{r}_{1,q,t}, \mathbf{r}_{2,q,t}, \mathbf{r}_{3,q,t},\mathbf{r}_{4,q,t})\}_{q=1,t}^{Q}$ containing multiple quartet data pairs sampled from existing task pool $\{\mathcal{D}_i\}_{i=1}^{t-1}$, where $Q$ denotes the number of quartet pairs in $|\mathcal{Q}_t|$.
\begin{itemize}
    \item $\mathbf{r}_{1,q,t}=(\mathbf{x}_{a,t},z_{a,t},y_t,e_t)$ with class $y_t$ and environment $e_t$
    \item $\mathbf{r}_{2,q,t}=(\mathbf{x}_{b,t},z_{b,t},y'_t,e_t)$ with class $y'_t$ and environment $e_t$
    \item $\mathbf{r}_{3,q,t}=(\mathbf{x}_{c,t},z_{c,t},y_t,e'_t)$ with class $y_t$ and environment $e'_t$
    \item $\mathbf{r}_{4,q,t}=(\mathbf{x}_{d,t},z_{d,t},y'_t,e'_t)$ with class $y'_t$ and environment $e'_t$
\end{itemize}
Notice that  $\mathbf{r}_{1,q,t}$ and $\mathbf{r}_{2,q,t}$ (same to $\mathbf{r}_{3,q,t}$ and $\mathbf{r}_{4,q,t}$) share the same environment label $e_t$ but different labels $y_t$ and $y'_t$. $\mathbf{r}_{1,q,t}$ and $\mathbf{r}_{3,q,t}$ (same to $\mathbf{r}_{2,q,t}$ and $\mathbf{r}_{4,q,t}$) share the same label $y_t$ but different environments $e_t$ and $e'_t$. We view $\mathbf{r}_{3,q,t}$ ($\mathbf{r}_{4,q,t}$) is an alternative pair to $\mathbf{r}_{1,q,t}$ ($\mathbf{r}_{2,q,t}$) with changing environments. For simplicity, we omit the subscripts $q$ and $t$.

Under \cref{assump:data-inv}, for $(\mathbf{r}_1,\mathbf{r}_2)$ and $(\mathbf{r}_3,\mathbf{r}_4)$ within the \textit{same environment but different labels}, the data reconstruction loss $\mathcal{L}_{recon}$ is given:
\begin{equation}
\small
\begin{aligned}
\label{eq:loss_recon}
    \mathcal{L}_{{recon}}^q= dist[\mathbf{x}_a,D(\mathbf{s}_a,\mathbf{v}_b,\boldsymbol{\theta}_t^d ) ] + dist[\mathbf{x}_c,D(\mathbf{s}_c,\mathbf{v}_d,\boldsymbol{\theta}_t^d ) ]
\end{aligned}
\end{equation}
where $\mathbf{s}_a=h_s(\mathbf{x}_a, \boldsymbol{\theta}_t^s)$, $\mathbf{s}_c=h_s(\mathbf{x}_c, \boldsymbol{\theta}_t^s)$, $\mathbf{v}_b=h_v(\mathbf{x}_b,\boldsymbol{\theta}_t^v)$, and $\mathbf{v}_d=h_v(\mathbf{x}_d,\boldsymbol{\theta}_t^v)$. $dist:\mathcal{X}\times\mathcal{X}\rightarrow\mathbb{R}$ indicates a distance metric, where we use $\ell_1$ norm in the experiments.

Similarly, under \cref{assump:class-inv}, for $(\mathbf{r}_1,\mathbf{r}_3)$ and $(\mathbf{r}_2,\mathbf{r}_4)$  with the \textit{same label but different environments}, the class invariance loss $\mathcal{L}_{inv}$ is given:
\begin{equation}
\small
\begin{aligned}
\label{eq:loss_inv}
    \mathcal{L}_{inv}^q=
    \ell_{CE}\Big(\omega\big(h_s(\mathbf{x}_{a\rightarrow c},\boldsymbol{\theta}^s_t),\boldsymbol{\theta}_t^{cls}\big),y\Big)
    +\ell_{CE}\Big(\omega\big(h_s(\mathbf{x}_{b\rightarrow d},\boldsymbol{\theta}^s_t),\boldsymbol{\theta}_t^{cls} \big),y'\Big)
\end{aligned}
\end{equation}
where $\mathbf{x}_{a\rightarrow c} = D(\mathbf{s}_a,\mathbf{v}_c,\boldsymbol{\theta}_t^d)$, $\mathbf{x}_{b\rightarrow d} = D(\mathbf{s}_b,\mathbf{v}_d,\boldsymbol{\theta}_t^d)$, and $\ell_{CE}:\mathbb{R}\times\mathbb{R}\rightarrow\mathbb{R}$ is the cross-entropy loss function.

Finally, to ensure prediction accuracy within a fair domain, we combine $(\mathbf{r}_{1,q}, \mathbf{r}_{2,q}, \mathbf{r}_{3,q}, \mathbf{r}_{4,q})$ together over the batch $\mathcal{Q}_t$ to estimate $\mathcal{L}_{cls}$ and $\mathcal{L}_{fair}$:
\begin{equation}
\small
\begin{aligned}
\label{eq:loss-cls-fair}
    \mathcal{L}_{cls}
    &=f_t(\mathcal{Q}_t, \boldsymbol{\theta}_t^s\oplus\boldsymbol{\theta}_t^{cls}) \\
    &=\sum\nolimits\nolimits_{q=1}^{Q}\sum\nolimits_{k}^{\{a,b,c,d\}}\ell_{CE}\big(\omega(h_s(\mathbf{x}_{k,q},\boldsymbol{\theta}_t^s),\boldsymbol{\theta}_t^{cls}),y_{k,q}\big)\\
    \mathcal{L}_{fair} &= \sum\nolimits_{q=1}^{{Q}} g(\mathcal{Q}_t, \boldsymbol{\theta}_t^s\oplus\boldsymbol{\theta}_t^{cls})
\end{aligned}
\end{equation}
where $\oplus$ denotes concatenation operator between $\boldsymbol{\theta}_t^{s}$ and $\boldsymbol{\theta}_t^{cls}$.

\begin{algorithm}[!t]
\caption{\sysname{}}
\label{alg:our-algor}
\begin{algorithmic}[1]
\State Input: batch size $Q$, learning rate $\eta_1, \eta_2$, margin $\epsilon_1,\epsilon_2, \epsilon_3$.
\State Randomly initialize $\boldsymbol{\theta}_{t=0}\in\Theta$ and $\lambda_{0,1},\lambda_{0,2},\lambda_{0,3}\in\mathbb{R}_+$
\State Initial the domain buffer as empty, $\mathcal{U}\leftarrow[\:]$.
\State Initial the task buffer as empty, $\mathcal{T}\leftarrow[\:]$.
\For{each $t\in[T]$}
    \State Record the performance of $(\boldsymbol{\theta}^s_{t-1},\boldsymbol{\theta}^{cls}_{t-1})$ on $\mathcal{D}_t$.
    \If{$e_t\notin \mathcal{U}$}
        \State $\mathcal{U}\leftarrow\mathcal{U}\cup\{e_t\}$
    \EndIf
    \State \multiline{%
        Assign $\boldsymbol{\theta}_t\leftarrow\boldsymbol{\theta}_{t-1},\:\lambda_{t,1}\leftarrow\lambda_{t-1,1},\:\lambda_{t,2}\leftarrow\lambda_{t-1,2},\:\lambda_{t,3}\leftarrow\lambda_{t-1,3}$}
    \For{$n=1,2\cdots$ steps}
        \If{$|\mathcal{U}|\neq 1$}
            \State \multiline{%
                Randomly sample a batch $\mathcal{Q}_t\subset\mathcal{T}$ indicated in \cref{sec:learning-dynamic}.}
            \State \multiline{%
                Compute $\mathcal{L}_{recon}^q$ and $\mathcal{L}_{inv}^q$ using \cref{eq:loss_recon,eq:loss_inv} for each quartet pair.}
            \State \multiline{%
                $\mathcal{L}_{recon} = \frac{1}{Q}\sum\nolimits_{q=1}^Q\mathcal{L}_{recon}^q$ and $\mathcal{L}_{inv} = \frac{1}{Q}\sum\nolimits_{q=1}^Q\mathcal{L}_{inv}^q$}
        \Else 
            \State \multiline{%
                Randomly sample a batch of doublet data pairs $\mathcal{Q}_t=\{((\mathbf{x}_{i,q,t},z_{i,q,t},y_{i,q,t},e), (\mathbf{x}_{j,q,t},z_{j,q,t},y_{j,q,t},e))\}_{q=1}^Q$, where $\mathcal{Q}_t\subset\mathcal{T}$.}
            \State \multiline{%
                Compute $\mathcal{L}_{recon}^q$ using \cref{eq:loss_recon} for each doublet pair.}
            \State $\mathcal{L}_{recon} = \frac{1}{Q}\sum\nolimits_{q=1}^Q\mathcal{L}_{recon}^q$
            \State Set $\mathcal{L}_{inv}=0$
        \EndIf
        \State Compute $\mathcal{L}_{cls}$, $\mathcal{L}_{fair}$ using \cref{eq:loss-cls-fair}.
        \State Compute $\mathcal{L}_{total}$ using \cref{eq:total-loss}.
        \State $\boldsymbol{\theta}_t^s\leftarrow\text{Adam}(\mathcal{L}_{total},\:\boldsymbol{\theta}_t^s,\:\eta_1)$
        \State $\boldsymbol{\theta}_t^v\leftarrow\text{Adam}(\lambda_{t,2}\cdot\mathcal{L}_{recon}+\lambda_{t,3}\cdot\mathcal{L}_{inv},\:\boldsymbol{\theta}_t^v,\:\eta_1)$
        \State $\boldsymbol{\theta}_t^d\leftarrow\text{Adam}(\lambda_{t,2}\cdot\mathcal{L}_{recon}+\lambda_{t,3}\cdot\mathcal{L}_{inv},\:\boldsymbol{\theta}_t^d,\:\eta_1)$
        \State \multiline{%
            $\boldsymbol{\theta}_t^{cls}\leftarrow\text{Adam}(\mathcal{L}_{cls}+\lambda_{t,1}\cdot\mathcal{L}_{fair}+\lambda_{t,3}\cdot\mathcal{L}_{inv},\:\boldsymbol{\theta}_t^{cls},\:\eta_1)$}
        \State $\lambda_{t,1}\leftarrow\max\big\{\lambda_{t,1}+\eta_2\cdot(\mathcal{L}_{fair}-\epsilon_1),\:0\big\}$
        \State $\lambda_{t,2}\leftarrow\max\big\{\lambda_{t,2}+\eta_2\cdot(\mathcal{L}_{recon}-\epsilon_2),\:0\big\}$
        \If{$|\mathcal{U}|\neq 1$}
            \State $\lambda_{t,3}\leftarrow\max\big\{\lambda_{t,3}+\eta_2\cdot(\mathcal{L}_{inv}-\epsilon_3),\:0\big\}$
        \EndIf
    \EndFor
    \State $\mathcal{T}\leftarrow\mathcal{T}\cup\{\mathcal{D}_t\}$
\EndFor
\end{algorithmic}
\end{algorithm}
\setlength{\textfloatsep}{3pt}

\subsection{A Practical Online Algorithm: \sysname{}}
In practice, requirements for remaining data invariance for datapoints sampled in the same environment with different labels and for keeping class invariance for datapoints sampled with the same label within various environments are hard to be satisfied. Similar to the fairness constraint, it is a strict equality constraint that is difficult to enforce in practice. To alleviate some of such difficulties, we relax the loss functions with empirical constants that
\begin{equation}
\small
\begin{aligned}
    &\mathcal{L}_{fair} \leq \epsilon_1\\
    &\mathcal{L}_{recon} = \frac{1}{Q}\sum\nolimits_{q=1}^Q\mathcal{L}_{recon}^q \leq \epsilon_2; \quad
    \mathcal{L}_{inv} = \frac{1}{Q}\sum\nolimits_{q=1}^Q\mathcal{L}_{inv}^q \leq \epsilon_3
\end{aligned}
\end{equation}
$\epsilon_1,\epsilon_2,\epsilon_3>0$ are fixed margins that control the extent to violations. 
\begin{equation}
    \small
\begin{aligned}
\label{eq:total-loss}
    \mathcal{L}_{total} = \mathcal{L}_{cls}+\lambda_{t,1}(\mathcal{L}_{fair}-\epsilon_1)+\lambda_{t,2}(\mathcal{L}_{recon}-\epsilon_2)+\lambda_{t,3}(\mathcal{L}_{inv}-\epsilon_3)
\end{aligned}
\end{equation}
\begin{sloppypar}
Furthermore, we propose a primal-dual \cref{alg:our-algor} for efficient optimization, wherein it alternates between optimizing $\boldsymbol{\theta}_t=\{\boldsymbol{\theta}_t^s, \boldsymbol{\theta}_t^v,\boldsymbol{\theta}_t^d,\boldsymbol{\theta}_t^{cls}\}$ at each time via minimizing the empirical Lagrangian with fixed dual $\lambda_t=\{\lambda_{t,1},\lambda_{t,2},\lambda_{t,3}\}$ corresponding for $\mathcal{L}_{fair}$, $\mathcal{L}_{recon}$ as well as $\mathcal{L}_{inv}$ and updating the dual variable according to the minimizer (lines 24-32). The primal-dual iteration has clear advantages over stochastic gradient descent in solving constrained optimization problems. Specifically, it avoids introducing extra balancing hyperparameters. Moreover, it provides convergence guarantees once we have sufficient iterations and a sufficiently small step size \cite{zhang2022towards}.
\end{sloppypar}
Moreover, because each task corresponds to a timestamp $t$ and an unknown environment before $\mathcal{D}_t$ arrives, the collected task pool $\{\mathcal{D}_i\}_{i=1}^{t-1}$ may be sampled from a single environment. In this sense, instead of using a batch stated in \cref{sec:learning-dynamic} with multiple quartet pairs, a sampled batch with doublet pairs containing $\{(\mathbf{r}_{1,q,t},\mathbf{r}_{1,q,t})\}_{q=1,t}^Q$ is considered. As a consequence, the class invariance loss in \cref{eq:loss_inv} is set to zero (lines 17-20).

%% file: analysis.tex
We first state assumptions about the online learning problem for changing environments that are largely used in \cite{zhang-nips-2018,Finn-ICML-2019,zhao-KDD-2022,zhang-2020-AISTATS}. Then we provide theoretical guarantees for the proposed \sysnameregret{} regarding the loss regret and violation of cumulative fair constraints.

\begin{assumption}[Bounded Parameter Domain]
\label{assump:bounded-parameter-domain}
The parameter domain $\Theta$ has a bounded diameter $D$ and contains the origin.
\begin{equation*}
    \small
\begin{aligned}
    \max_{\boldsymbol{\theta}_1,\boldsymbol{\theta}_2\in\Theta}||\boldsymbol{\theta}_1-\boldsymbol{\theta}_2||\leq D,\quad \forall \boldsymbol{\theta}_1,\boldsymbol{\theta}_2\in\Theta
\end{aligned}
\end{equation*}
\end{assumption}

\begin{assumption}[Convexity]
\label{assump:convexity}
Domain $\Theta$ is convex and closed. The loss function $f_t$ and the fair function $g$ are convex.
\end{assumption}

\begin{assumption}[$F-$Lipschitz]
\label{assump:F-lipschitz}
There exists a positive constant $F$ such that 
\begin{equation*}
    \small
\begin{aligned}
    &\max_{\boldsymbol{\theta}_1,\boldsymbol{\theta}_2\in\Theta}|f_t(\cdot,\boldsymbol{\theta}_1)-f_t(\cdot,\boldsymbol{\theta}_2)|\leq F,\\
    &\max_{\boldsymbol{\theta}_1\in\Theta}||g(\cdot,\boldsymbol{\theta}_1)||\leq F, \quad \forall \boldsymbol{\theta}_1,\boldsymbol{\theta}_2\in\Theta, \forall t\in[T]
\end{aligned}
\end{equation*}
\end{assumption}

\begin{assumption}[Bounded gradient]
\label{assump:bounded-gradient}
\begin{sloppypar}
The gradients $\nabla f_t(\boldsymbol{\theta})$ and $\nabla g(\boldsymbol{\theta})$ exist, and 
they are bounded by a positive constant $G$ on $\Theta$, \textit{i.e.,}
\end{sloppypar}
\begin{equation*}
    \small
\begin{aligned}
    \max_{\boldsymbol{\theta}\in\Theta}||\nabla f_t(\cdot,\boldsymbol{\theta})||\leq G, \quad \max_{\boldsymbol{\theta}\in\Theta}||\nabla g(\cdot,\boldsymbol{\theta})||\leq G, \quad \forall \boldsymbol{\theta}\in\Theta,\forall t\in[T]
\end{aligned}    
\end{equation*}
\end{assumption}

Examples where these assumptions hold include logistic regression and $L2$ regression over a bounded domain. As for constraints, a family of fairness notions, such as DDP stated in \cref{def:fairness-notion} of \cref{sec:group-fairness}, are applicable as discussed in \cite{Lohaus-2020-ICML}. For simplicity, in this section, we omit $\mathcal{D}$ used in $f_t,\forall t$ and $g$.

\begin{sloppypar}
As introduced in \cref{sec:settings and problem formulation}, a sequence of parameters $\{\boldsymbol{\theta}_1^s,\cdots,\boldsymbol{\theta}_T^s,$ $\boldsymbol{\theta}_1^{cls},\cdots,\boldsymbol{\theta}_T^{cls}\}$ generated by the learner are evaluated with comparator sequence $\{\mathbf{u}_1^s,\cdots,\mathbf{u}_T^s,\boldsymbol{\theta}^{cls}\}$ in \sysnameregret{}. We claim that \sysnameregret{} takes a mixed form of the static and dynamic regrets with respect to $\{\boldsymbol{\theta}_t^{cls}\}_{t=1}^T$ and $\{\boldsymbol{\theta}_t^{s}\}_{t=1}^T$, respectively. Since the comparator $\boldsymbol{\theta}^{cls}$ in the static regret is performed as the best fixed one in hindsight, intuitively the comparator sequence can be extended to $\{\mathbf{u}_1^s,\cdots,\mathbf{u}_T^s,\boldsymbol{\theta}^{cls},\cdots,\boldsymbol{\theta}^{cls}\}$ by making $T$ copies of $\boldsymbol{\theta}^{cls}$. For simplicity, we denote the sequence of the learner's parameters and comparators as $\{\boldsymbol{\theta}_t^l\}_{t=1}^T$ and $\{\mathbf{u}_t^c\}_{t=1}^T$, respectively, where $\boldsymbol{\theta}_t^l:=\boldsymbol{\theta}_t^s\oplus\boldsymbol{\theta}_t^{cls}$ and $\mathbf{u}_t^c:=\mathbf{u}_t^s\oplus\boldsymbol{\theta}^{cls}$, $\forall t\in[T]$.
\end{sloppypar}

Furthermore, different from the static regret introduced in \cref{eq:static-regret}, it is impossible to achieve a sub-linear upper bound using dynamic regret in general. Instead, we can bound the dynamic regret in terms of some certain regularity of the comparator sequence or the function sequence, such as the path-length \cite{Zinkevich-ICML-2003} which measures the temporal variability of the comparator sequence.
\begin{equation}
    \small
\begin{aligned}
\label{eq:path-length}
    P_T = \sum\nolimits_{t=1}^T || \mathbf{u}_{t+1}^c-\mathbf{u}_t^c||_2
\end{aligned}
\end{equation}
Finally, under \cref{assump:bounded-gradient,assump:bounded-parameter-domain,assump:F-lipschitz,assump:convexity} and \cref{eq:path-length}, we state the key \cref{theorem:regret-bound} that the proposed \sysnameregret{} enjoys theoretic guarantees for both loss regret and violation of the long-term fairness constraint in the long run for \cref{alg:our-algor}.

\begin{table*}[!t]
\caption{Comparison of upper bounds in loss regret and constraint violations for changing environments across methods.}
\vspace{-3mm}
\setlength\tabcolsep{3pt}
\begin{tabular}{cccccc|c}
    \hline
    \multicolumn{1}{c|}{Algorithms}  & \textit{M. Zinkevich} \cite{Zinkevich-ICML-2003} & Ader \cite{zhang-nips-2018} & \multicolumn{1}{c}{AOD \cite{zhang-2020-AISTATS}} & \multicolumn{1}{c}{CBCE\cite{Jun-2017-AISTATS}} & FairSAOML \cite{zhao-KDD-2022} & \sysnameregret{} (Ours)\\ 
    \hline
    \multicolumn{1}{c|}{Loss Regret} & $\mathcal{O}(T^{1/2}(1+P_T))$ & $\mathcal{O}((T(1+P_T))^{1/2})$ & \multicolumn{1}{c}{$O\big((\tau\log T)^{1/2}\big)$} & \multicolumn{1}{c}{$O\big((\tau\log T)^{1/2}\big)$} & $O\big((\tau\log T)^{1/2}\big)$ & $\mathcal{O}((T(1+P_T))^{1/2})$\\ 
    \hline
    \multicolumn{1}{c|}{\begin{tabular}[c]{@{}c@{}} Constraint Violations\end{tabular}} & - & - & \multicolumn{1}{c}{-} & \multicolumn{1}{c}{-} & $O\big((\tau T\log T)^{1/4}\big)$ & $O(T^{1/2})$\\ 
    \hline
\end{tabular}
\label{tab:analysis-comparison}
\vspace{-4mm}
\end{table*}

\begin{theorem}
\label{theorem:regret-bound}
Suppose \cref{assump:bounded-gradient,assump:bounded-parameter-domain,assump:F-lipschitz,assump:convexity} hold, let $\{\boldsymbol{\theta}_t^s,\boldsymbol{\theta}_t^{cls}\}_{t=1}^T$ be the sequence generated by the online learner in \cref{alg:our-algor} and $\{\mathbf{u}_t^s\}_{t=1}^T\cup \{\boldsymbol{\theta}^{cls}\}$ be the comparator sequence, 
setting adaptive learning rates with
\begin{equation*}
    \small
\begin{aligned}
    \eta_{1,t} = \eta_{1,0}/\sqrt{T},\quad \eta_{2,t}=\eta_{2,0}/\sqrt{\eta_{1,t}}, \forall t\in[T]
\end{aligned}
\end{equation*}
where $\eta_{1,0}>0$ and $\eta_{2,0}\in(0, \frac{1}{\sqrt{2}G})$ are constants. We have
\begin{equation*}
    \small
\begin{aligned}
    &\sum_{t=1}^T f_t(h_s(\boldsymbol{\theta}_t^s),\boldsymbol{\theta}^{cls}_t) -\min_{\boldsymbol{\theta}^{cls}\in\Theta} \sum_{t=1}^T f_t(h_s(\mathbf{u}_t^s),\boldsymbol{\theta}^{cls})=\mathcal{O}\Big(\sqrt{T(1+P_T)}\Big)\\
    &\sum_{t=1}^T \Big|\Big|\big[g(h_s(\boldsymbol{\theta}_t^s),\boldsymbol{\theta}^{cls}_t)\big]_+\Big|\Big| = \mathcal{O}(\sqrt{T})
\end{aligned}
\end{equation*}
\end{theorem}
\begin{proof}
Proof of \cref{theorem:regret-bound} is given in \cref{app:proof}.
\end{proof}

\textbf{Discussion.} 
Under \cref{assump:bounded-gradient,assump:bounded-parameter-domain,assump:F-lipschitz,assump:convexity}, we provide comparable bounds for \sysnameregret{} with respect to both loss regret and violation of fair constraints. \cref{tab:analysis-comparison} lists a number of state-of-the-art works focusing on the problem of online learning in changing environments, where ours are added at the end. AOD \cite{zhang-2020-AISTATS}, CBCE \cite{Jun-2017-AISTATS}, and FairSAOML \cite{zhao-KDD-2022} address this problem by proposing strongly adaptive regret. In contrast to dynamic regret, strongly adaptive regret handles changing environments from a local perspective by proposing a set of intervals ranging from $\tau$ tasks. Ader \cite{zhang-nips-2018} and \textit{M. Zinkevich} \cite{Zinkevich-ICML-2003} tackle this problem using dynamic regret using the length-path regularity in \cref{eq:path-length}. Although the loss regret we derived for \sysnameregret{} is comparable to the one in Ader, the latter ignores the long-term fair constraint which is essential for fair online learning.

%% file: experiments.tex
\begin{figure*}[!t]
    \centering
    \includegraphics[width=0.8\linewidth]{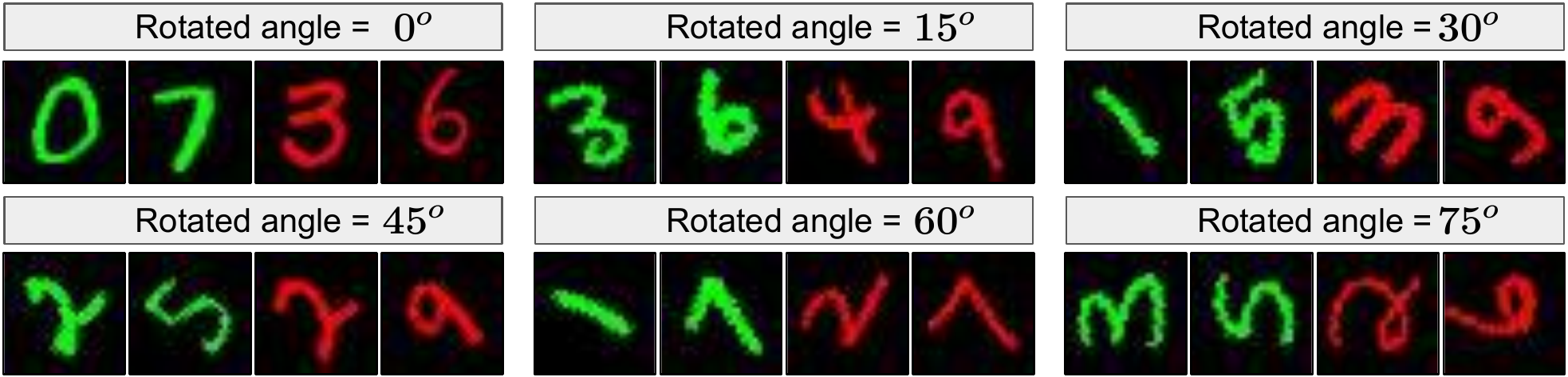}
    \vspace{-3mm}
    \caption{Visualization of the Rotated-Colored-MNIST dataset. }
    \vspace{-3mm}
    \label{fig:rcMNIST-vis}
\end{figure*}

In previous sections, we derive a theoretically principled algorithm assuming convexity everywhere. However, it has been known that deep learning models provide advanced performance in real-world applications, but they have a non-convex landscape with challenging theoretical analysis. Taking inspiration from the success of deep learning, we empirically evaluate the proposed algorithm \sysname{} using neural networks in this section.

\textbf{Datasets.}
We consider four datasets: Rotated-Colored-MNIST (rcMNIST), New York Stop-and-Frisk \cite{Koh-icml-2021}, Chicago Crime \cite{Zhao-ICDM-2019}, and German Credit \cite{german-data} to evaluate our \sysname{} against state-of-the-art baselines, where rcMNIST is an image data and the other three are tabular datasets. We include the visualization of rcMNIST in \cref{fig:rcMNIST-vis}. 
\textbf{(1) Rotated-Colored-MNIST} is extended from the Rotated-MNIST dataset \cite{ghifary2015domain}, which consists of 10,000 digits from 0 to 9 with different rotated angles where environments are determined by angles $\{0, 15, 30, 45, 60, 75\}$. For simplicity, we consider binary classification where digits are labeled with 0 and 1 for digits from 0-4 and 5-9, respectively. For fairness concerns, each image has a green or red digit color as the sensitive attribute. We intentionally make correlations between labels and digit colors for each corresponding environment ranging from $\{0.9, 0.7, 0.5, 0.3, 0.1, 0.05\}$. We further divide data from each environment equally into 3 subsets, where each is considered a task. For 6 environments, there is are total of 18 tasks and each arrives one after another over time in order.
\textbf{(2) New York Stop-and-Frisk} \cite{Koh-icml-2021} is a real-world dataset on policing in New York City in 2011.
It documents whether a pedestrian who was stopped on suspicion of weapon possession would in fact possess a weapon. We consider race (\textit{i.e.,} black and non-black) as the sensitive label for each datapoint. Since this data consists of data from 5 cities in New York City, Manhattan, Brooklyn, Queens, Bronx, and Staten, data collected from each city is considered as an individual environment. To adapt to the setting of online learning, data in each environment is further split into 3 tasks, 15 tasks in total, where each task corresponds to a month's set of data of a city. 
\textbf{(3) Chicago Crime} \cite{Zhao-ICDM-2019} dataset contains information including demographics information (\textit{e.g.,} race, gender, age, population, \textit{etc.}), household, education, unemployment status, \textit{etc}. We use race (\textit{i.e.,} black and non-black) as the sensitive label. It consists of 16 tasks and each corresponds to a county of Chicago city as an environment. 
This dataset is initially used for multi-task fair regression learning in \cite{Zhao-ICDM-2019}, where crime counts are used as continuous labels for data records. In our experiments, we categorize crime counts into binary labels, high ($\geq 6$) and low ($<6$).
\textbf{(4) German Credit} \cite{german-data} dataset contains 1000 datapoints with 20 features. Gender (\textit{i.e.,} male and female) is used as sensitive attribute and credit risk (\textit{i.e.,} good and bad) is the target. Following \cite{zhao-KDD-2022, Wan_2021_AAAI}, to generate dynamic environments, we construct a larger dataset by combining three copies of the original data and flipping the original values of non-sensitive attributes by multiplying -1 for the middle copy. Therefore, each copy is considered as an environment. Each data copy is split into 2 tasks by time and there are 6 tasks in total.

\textbf{Evaluation Metrics.}
Three popular evaluation metrics to estimate fairness are used and each allows quantifying the extent to model bias.
\begin{itemize}[leftmargin=*]
    \item \textit{Demographic Parity} (DP) \cite{Dwork-2011-CoRR} is formalized as
        \[ \text{DP} =
            \begin{cases} 
                \mathbb{P}(\hat{Y}=1|Z=-1) \Big/ \mathbb{P}(\hat{Y}=1|Z=1), & \text{if DP}\leq 1 \\
                \mathbb{P}(\hat{Y}=1|Z=1) \Big/ \mathbb{P}(\hat{Y}=1|Z=-1), & \text{otherwise} 
            \end{cases}
        \]
        This is also known as a lack of disparate impact \cite{Feldman-KDD-2015}. A value closer to 1 indicates fairness.
    \item \textit{Equalized Odds} (EO) \cite{Hardt-NIPS-2016} is formalized as
            \[ \text{EO} =
            \begin{cases} 
                \mathbb{P}(\hat{Y}=1|Z=-1, Y=1) \Big/ \mathbb{P}(\hat{Y}=1|Z=1, Y=1), & \text{if EO}\leq 1 \\
                \mathbb{P}(\hat{Y}=1|Z=1, Y=1) \Big/ \mathbb{P}(\hat{Y}=1|Z=-1, Y=1), & \text{otherwise} 
            \end{cases}
        \]
        EO requires that $\hat{Y}$ has equal true positive and false negative rates between subgroups $z=-1$ and $z=1$. Same to DP, a value closer to 1 indicates fairness.
    \item \textit{Mean Difference} (MD) \cite{Zemel-ICML-2013} is a form of statistical parity, applied to the classification decisions, measuring the difference in the proportion of positive class of individuals in sub-groups.
        \begin{align*}
            \text{MD} = \Big|\frac{\sum_{i:z_i=1}\hat{y}_i}{\sum_{i:z_i=1}1}-\frac{\sum_{i:z_i=-1}\hat{y}_i}{\sum_{i:z_i=-1}1}\Big|
        \end{align*}
        A value closer to 0 indicates fairness.
\end{itemize}

\begin{figure*}[!t]
\captionsetup[subfigure]{aboveskip=-1pt,belowskip=-1pt}
\centering
    \begin{subfigure}[b]{0.245\textwidth}
        \includegraphics[width=\textwidth]{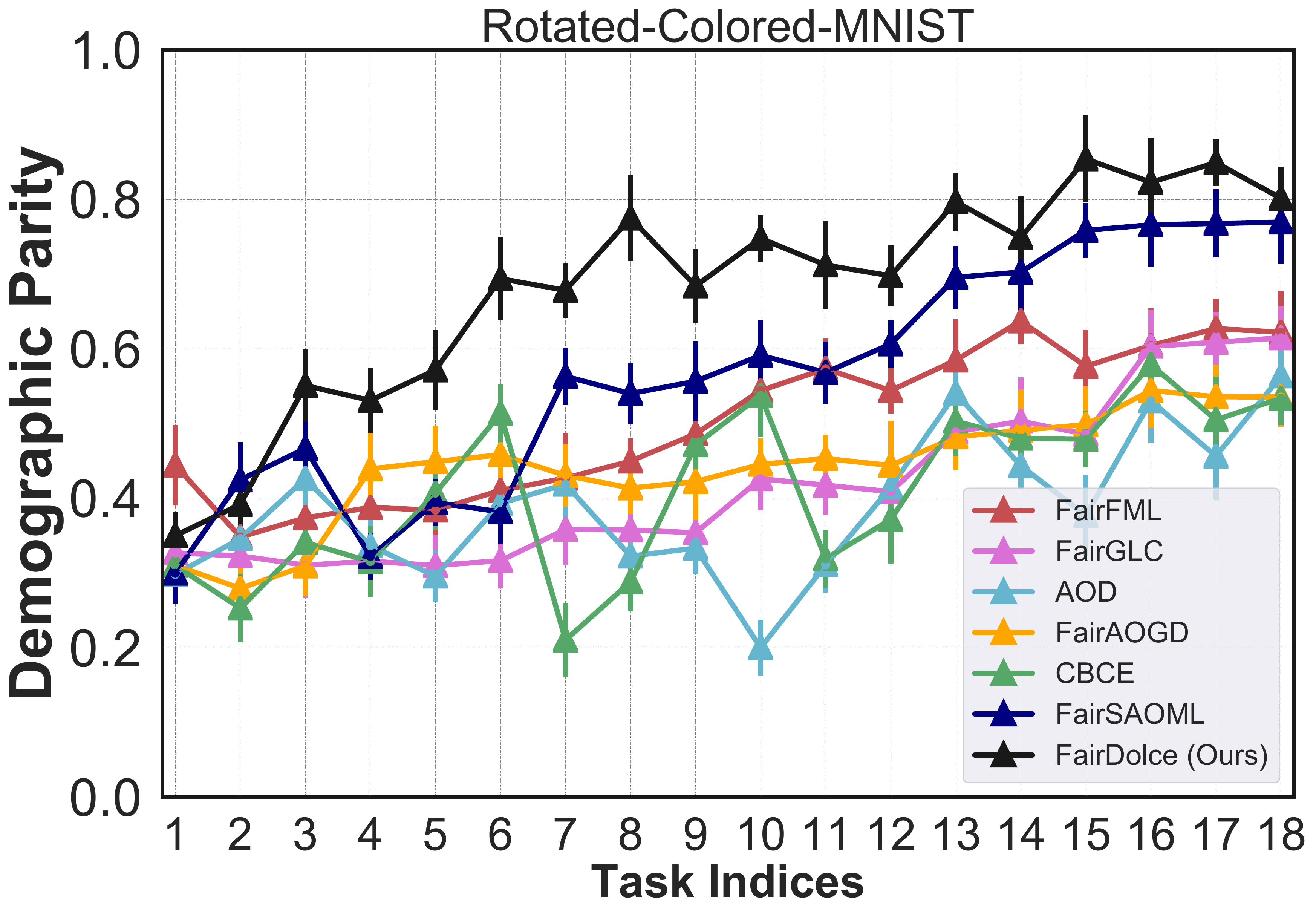}
        \caption{}
    \end{subfigure}
    \begin{subfigure}[b]{0.245\textwidth}
        \includegraphics[width=\textwidth]{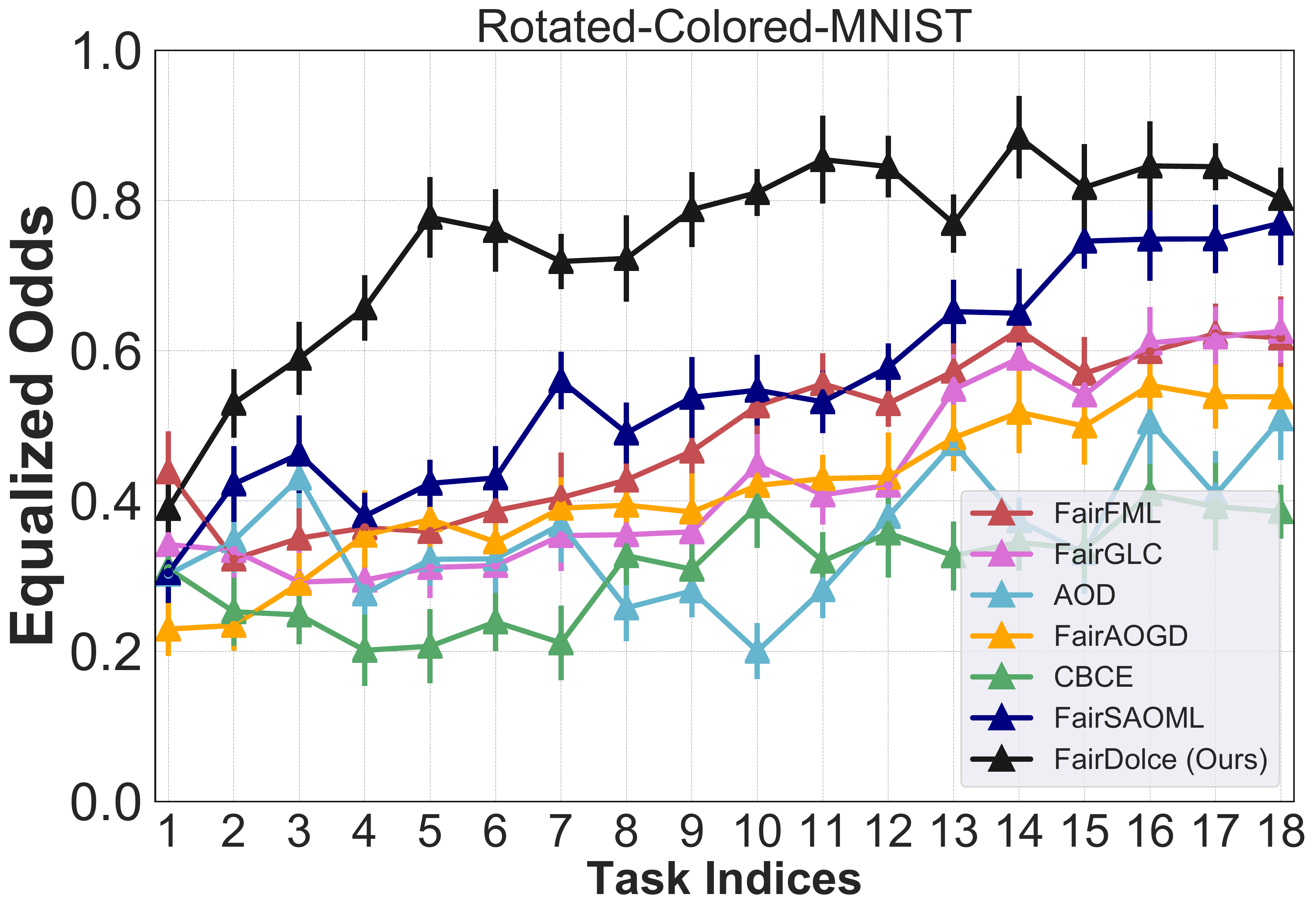}
        \caption{}
    \end{subfigure}
    \begin{subfigure}[b]{0.245\textwidth}
        \includegraphics[width=\textwidth]{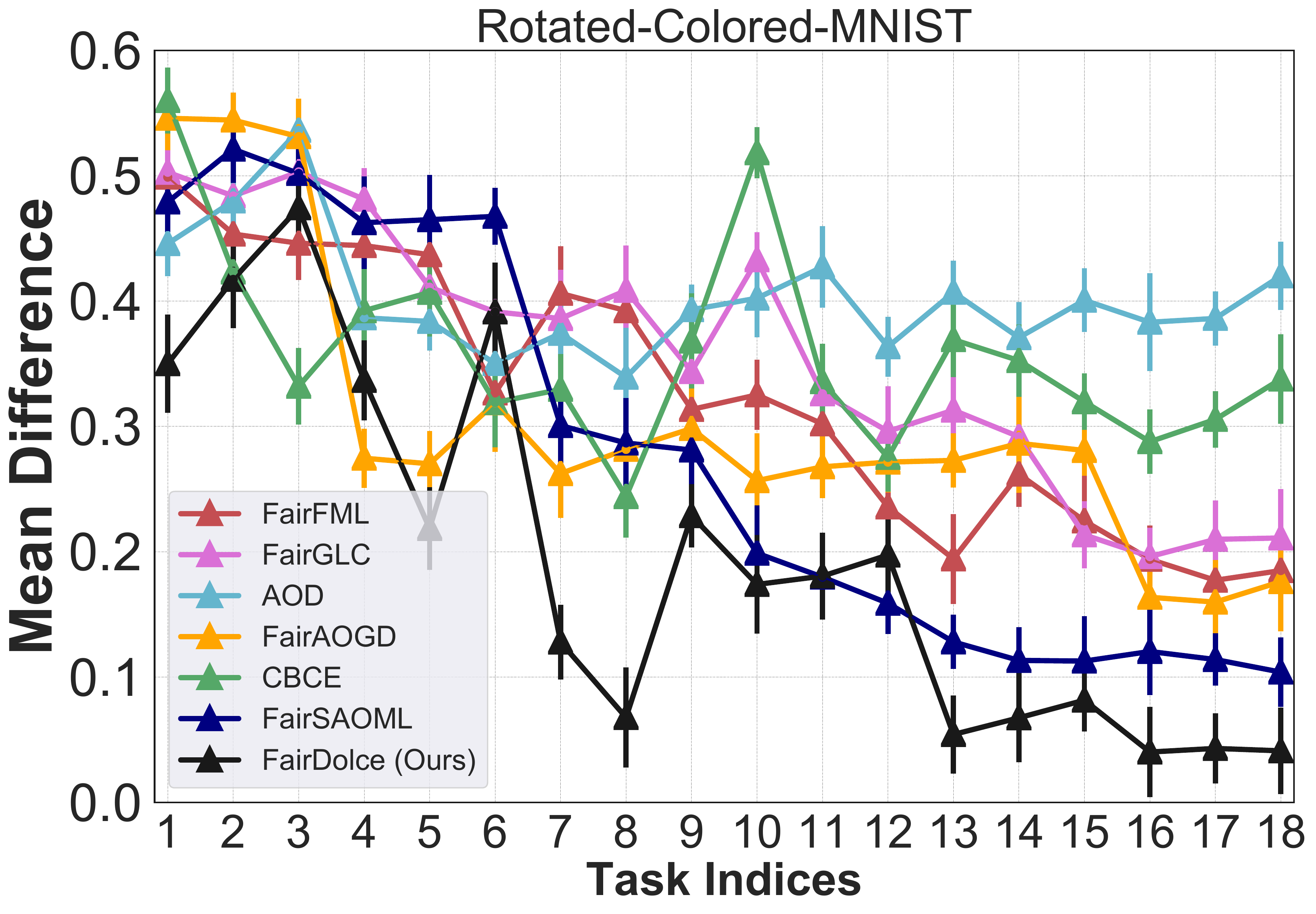}
        \caption{}
    \end{subfigure}
    \begin{subfigure}[b]{0.245\textwidth}
        \includegraphics[width=\textwidth]{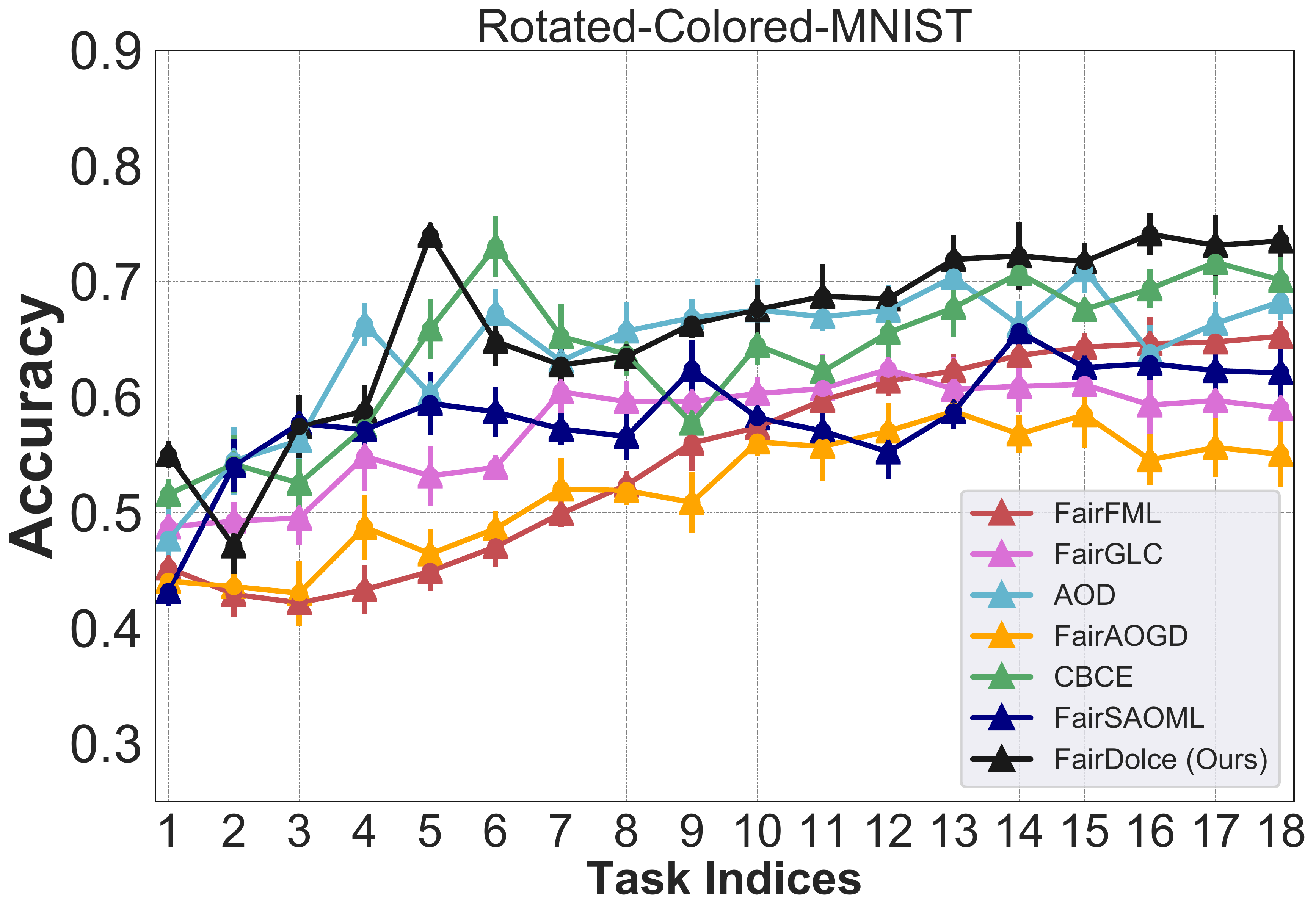}
        \caption{}
    \end{subfigure}

    \begin{subfigure}[b]{0.245\textwidth}
        \includegraphics[width=\textwidth]{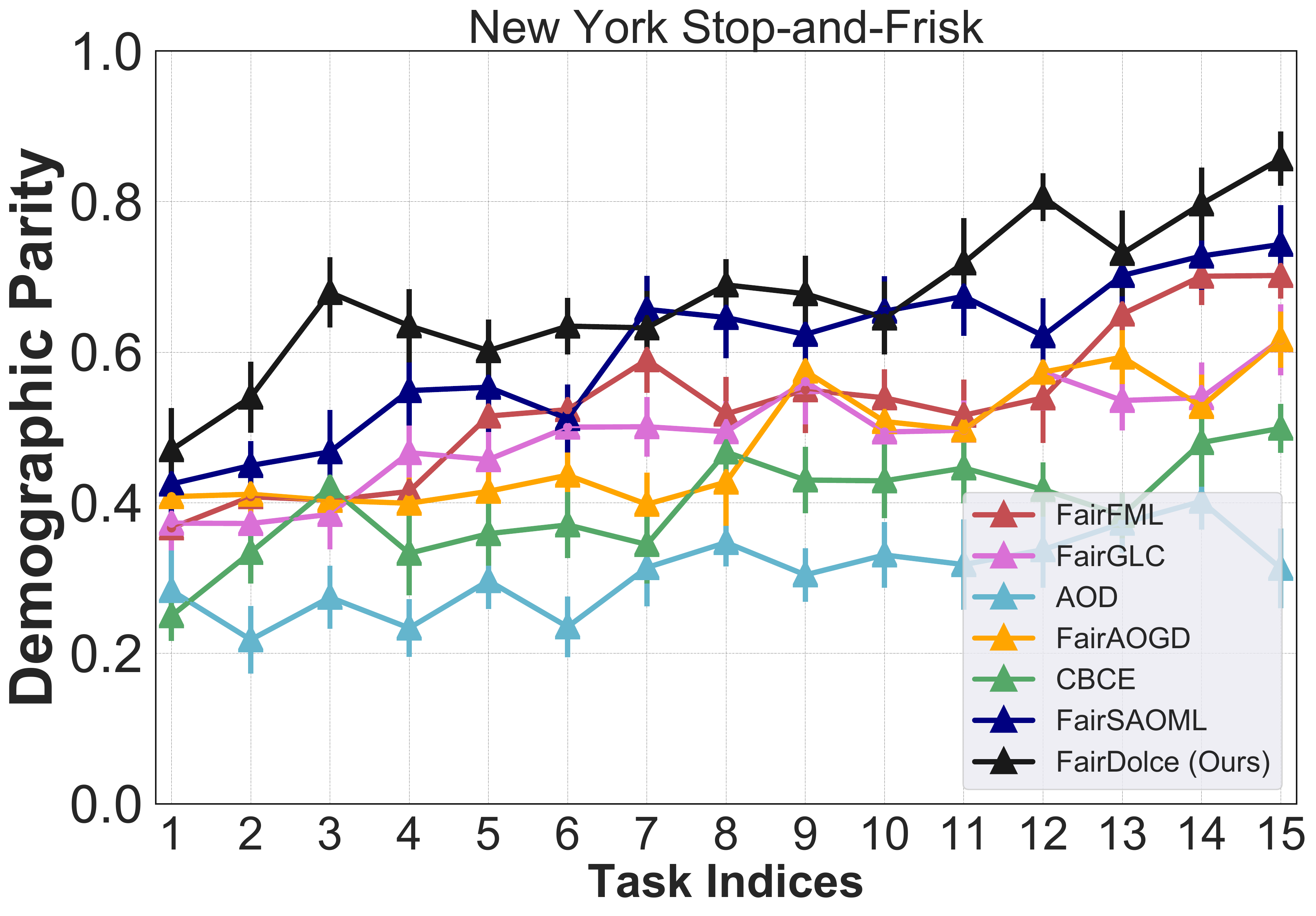}
        \caption{}
    \end{subfigure}
    \begin{subfigure}[b]{0.245\textwidth}
        \includegraphics[width=\textwidth]{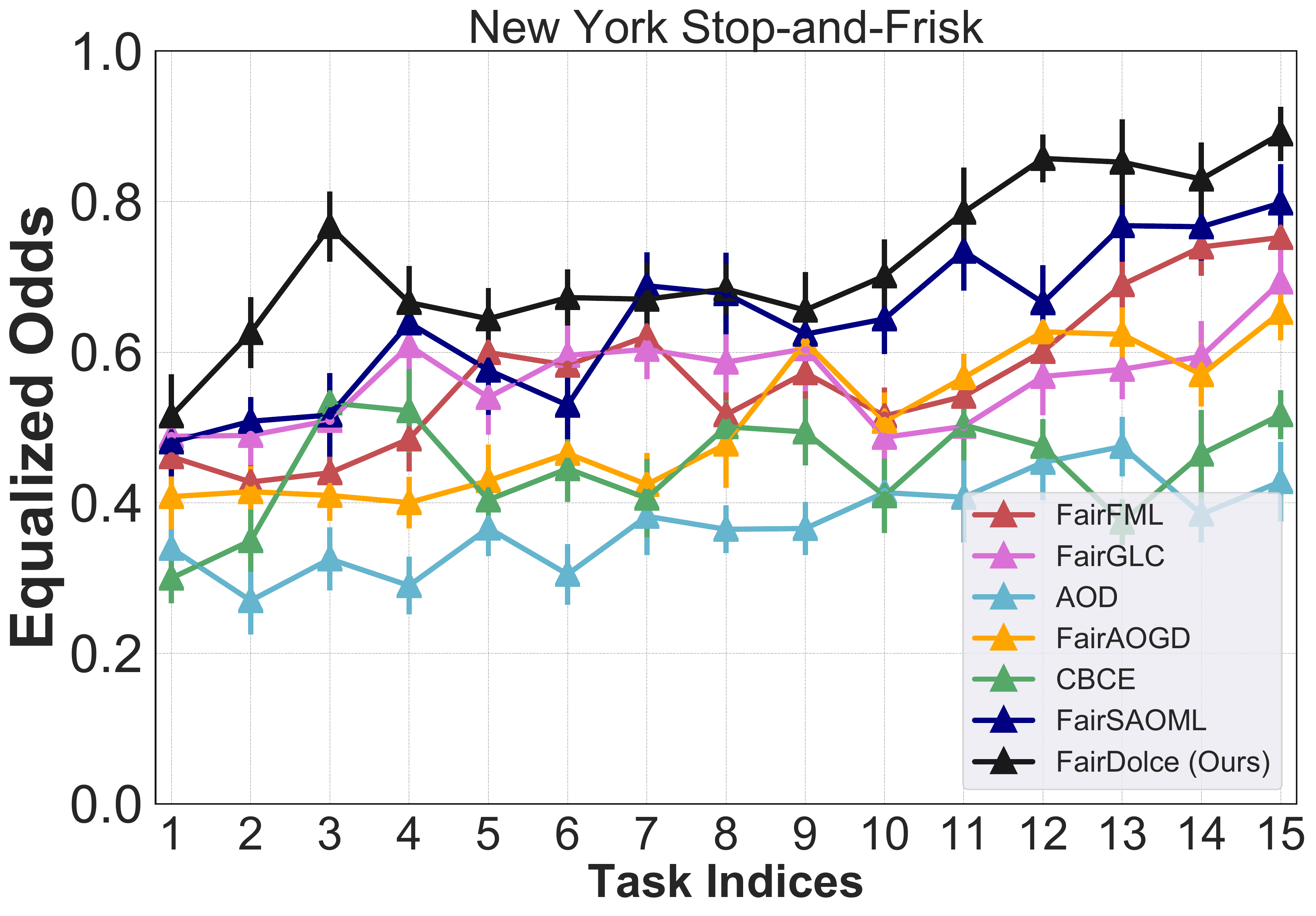}
        \caption{}
    \end{subfigure}
    \begin{subfigure}[b]{0.245\textwidth}
        \includegraphics[width=\textwidth]{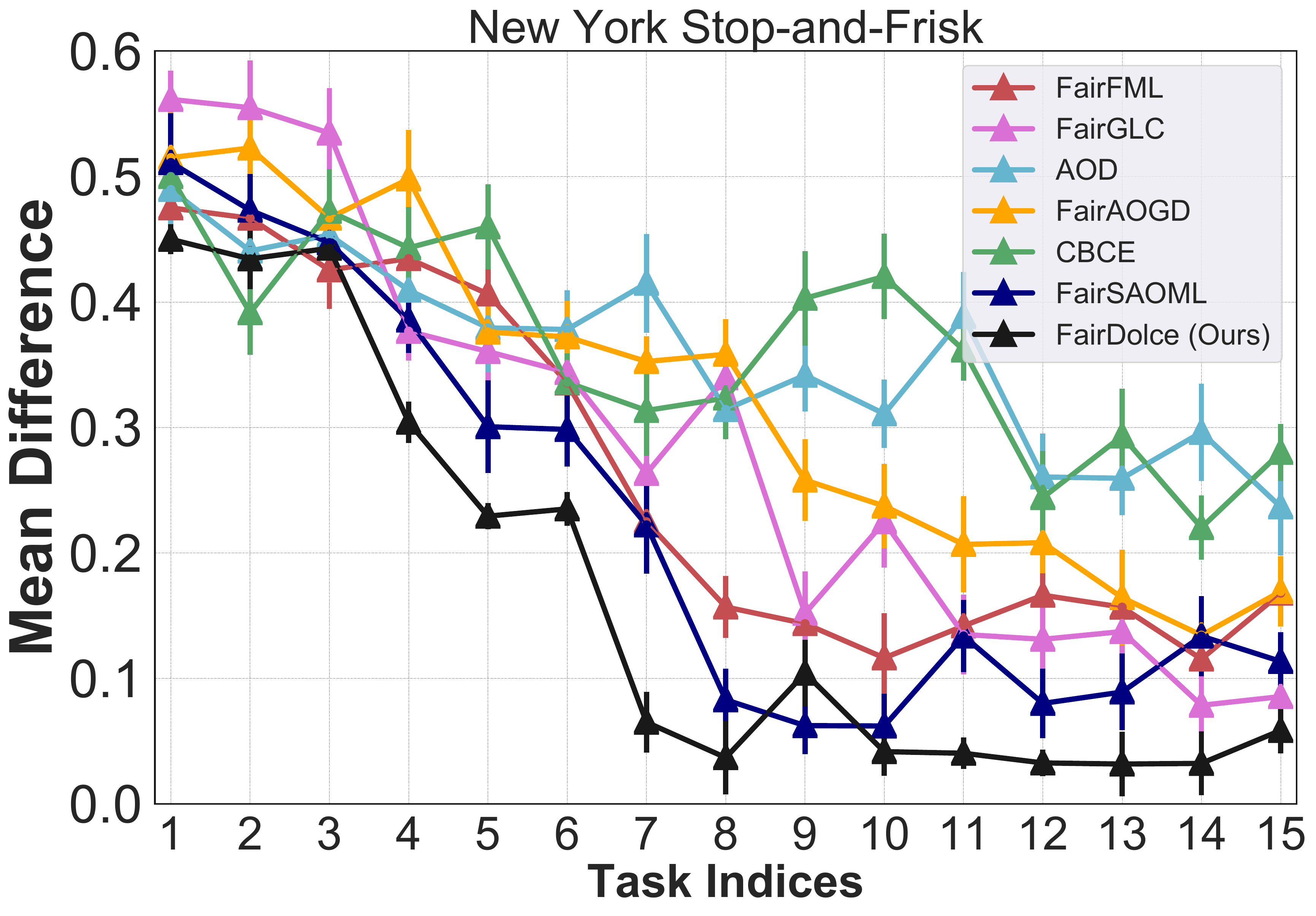}
        \caption{}
    \end{subfigure}
    \begin{subfigure}[b]{0.245\textwidth}
        \includegraphics[width=\textwidth]{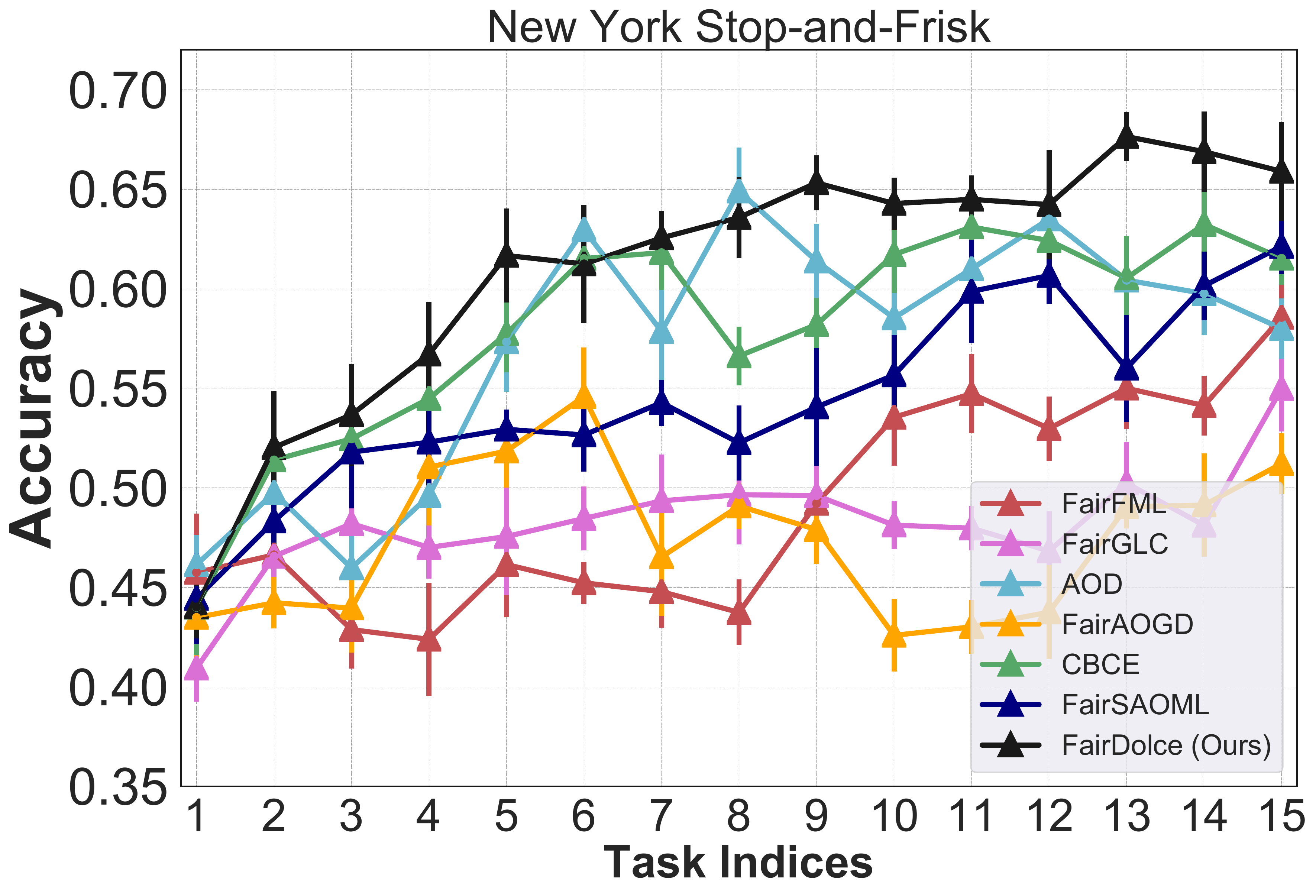}
        \caption{}
    \end{subfigure}

    \begin{subfigure}[b]{0.245\textwidth}
        \includegraphics[width=\textwidth]{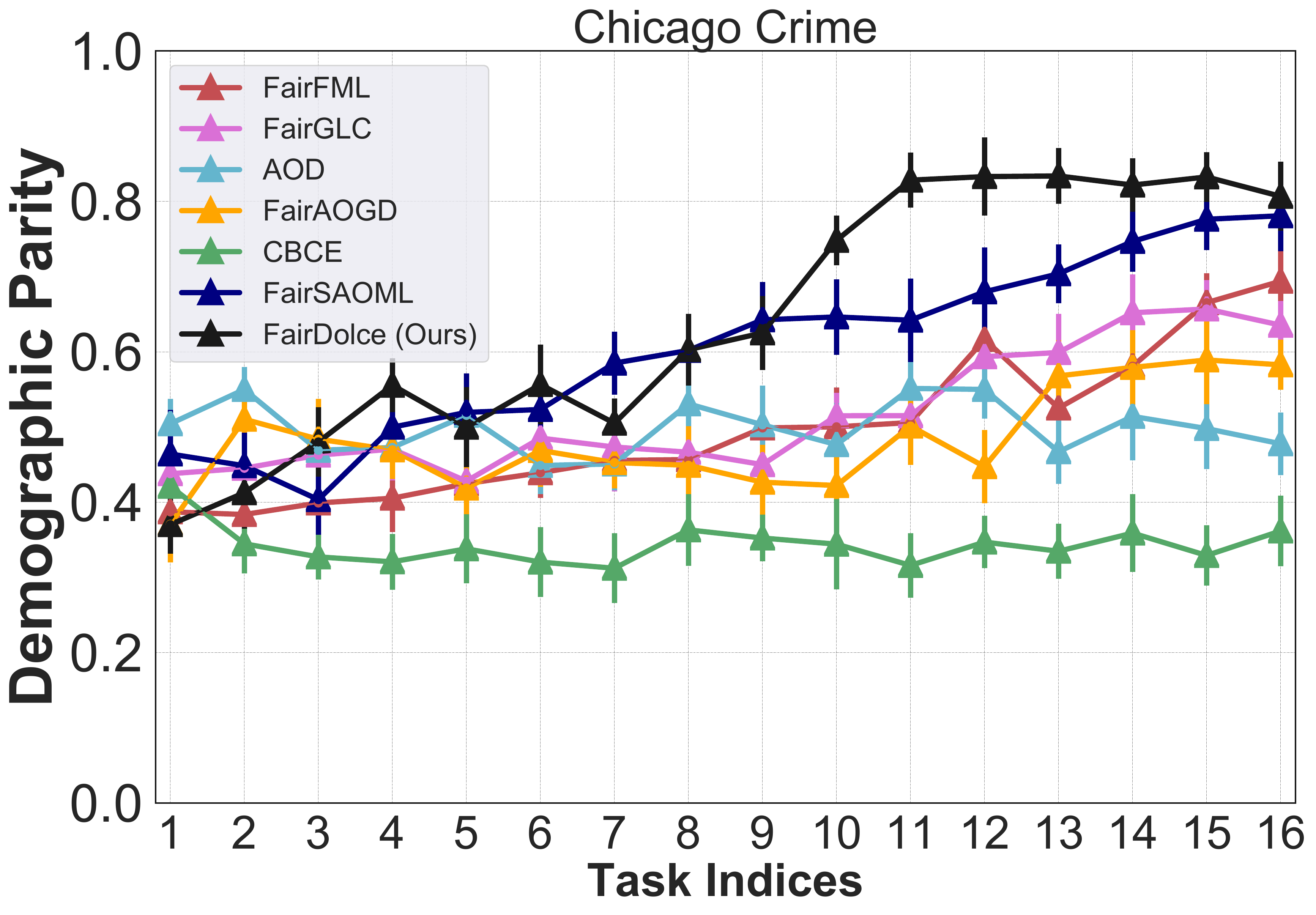}
        \caption{}
    \end{subfigure}
    \begin{subfigure}[b]{0.245\textwidth}
        \includegraphics[width=\textwidth]{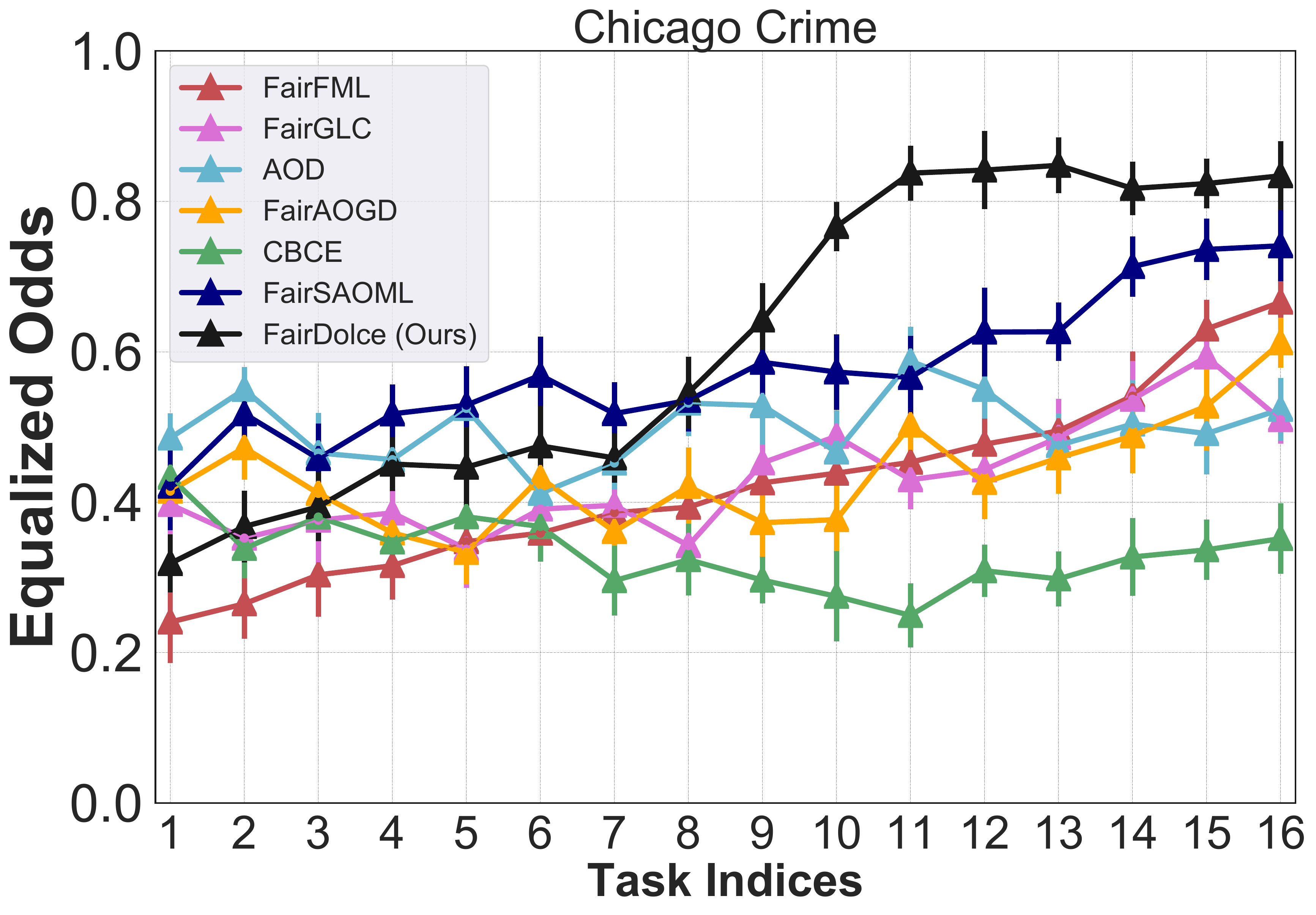}
        \caption{}
    \end{subfigure}
    \begin{subfigure}[b]{0.245\textwidth}
        \includegraphics[width=\textwidth]{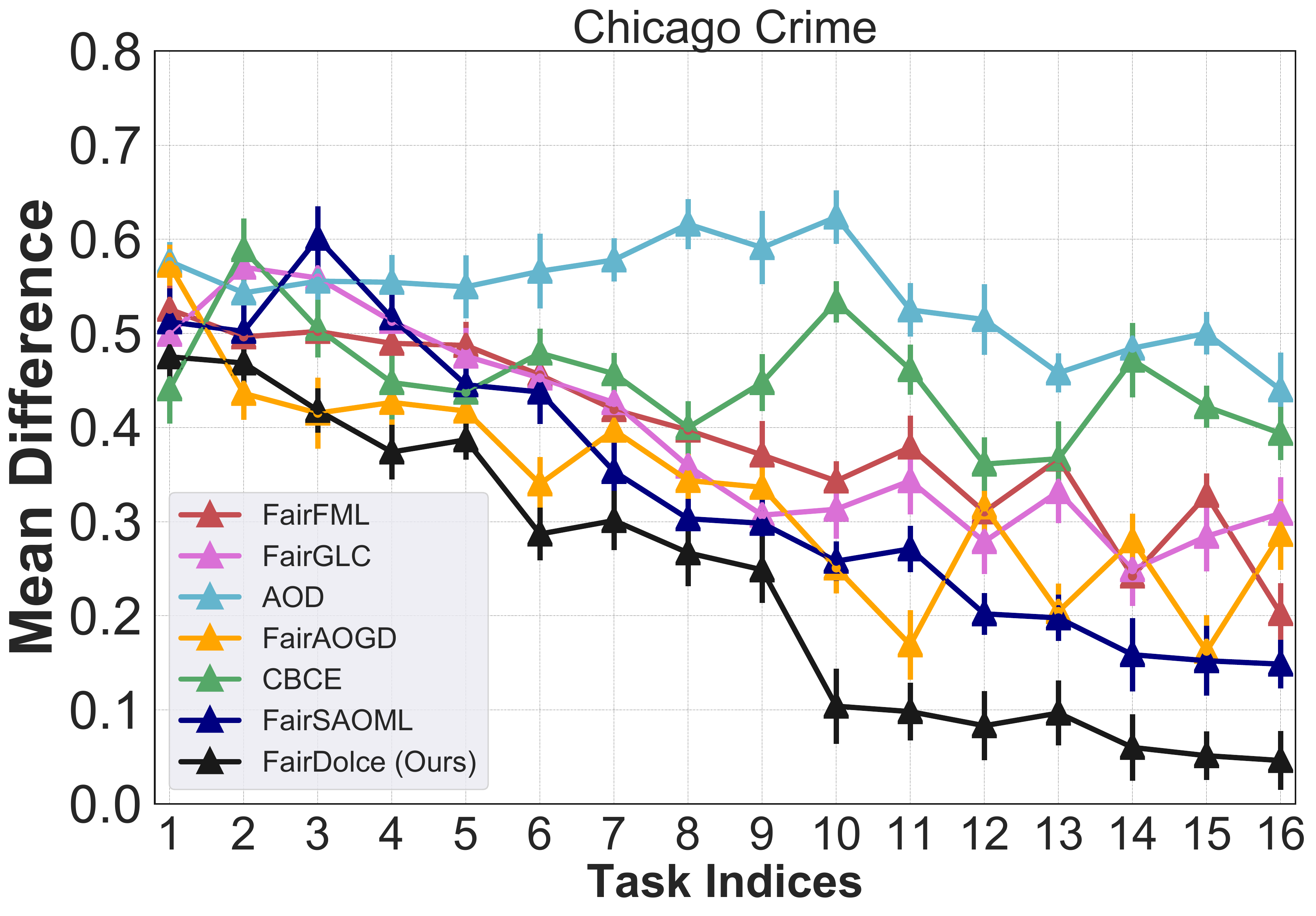}
        \caption{}
    \end{subfigure}
    \begin{subfigure}[b]{0.245\textwidth}
        \includegraphics[width=\textwidth]{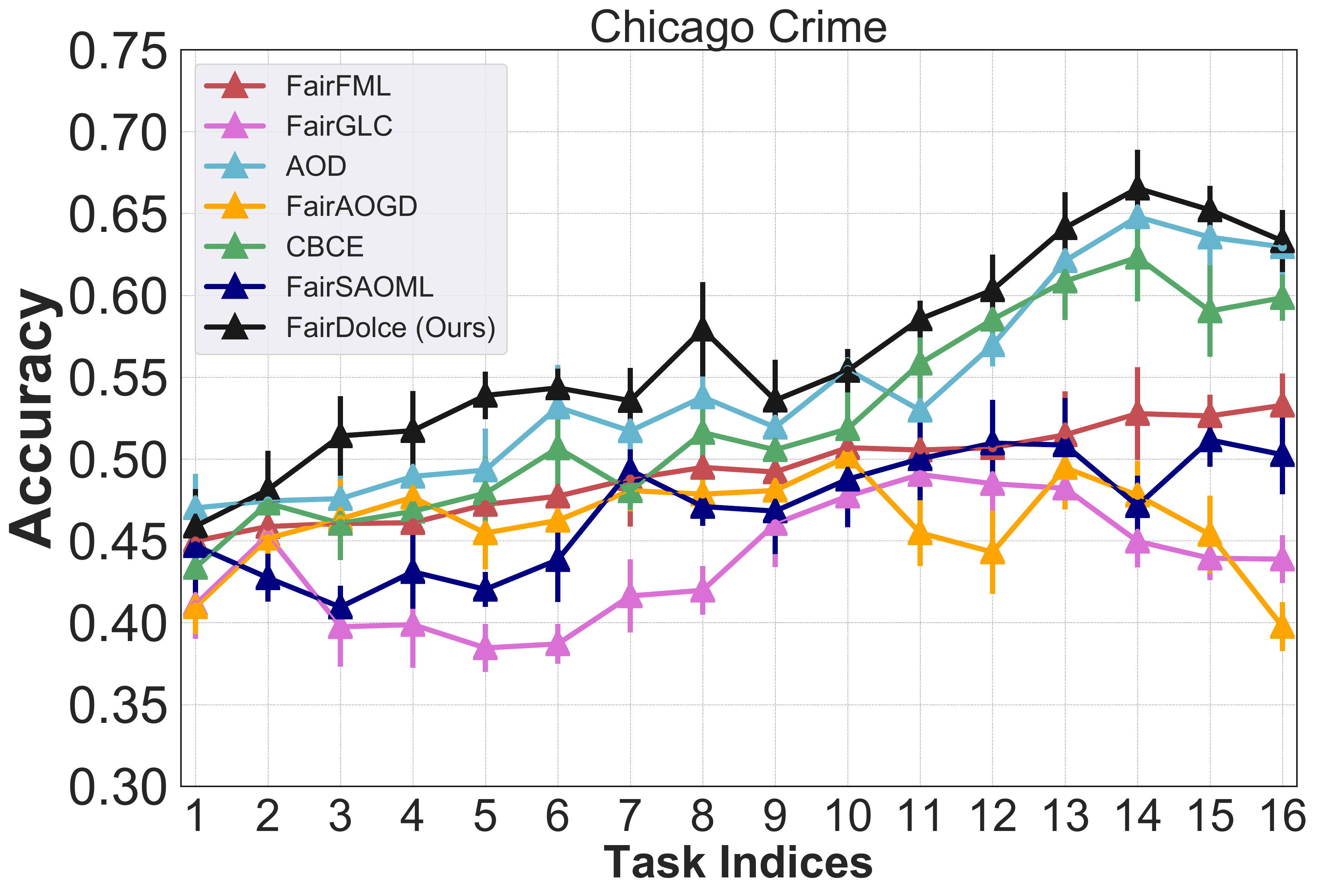}
        \caption{}
    \end{subfigure}

    \begin{subfigure}[b]{0.245\textwidth}
        \includegraphics[width=\textwidth]{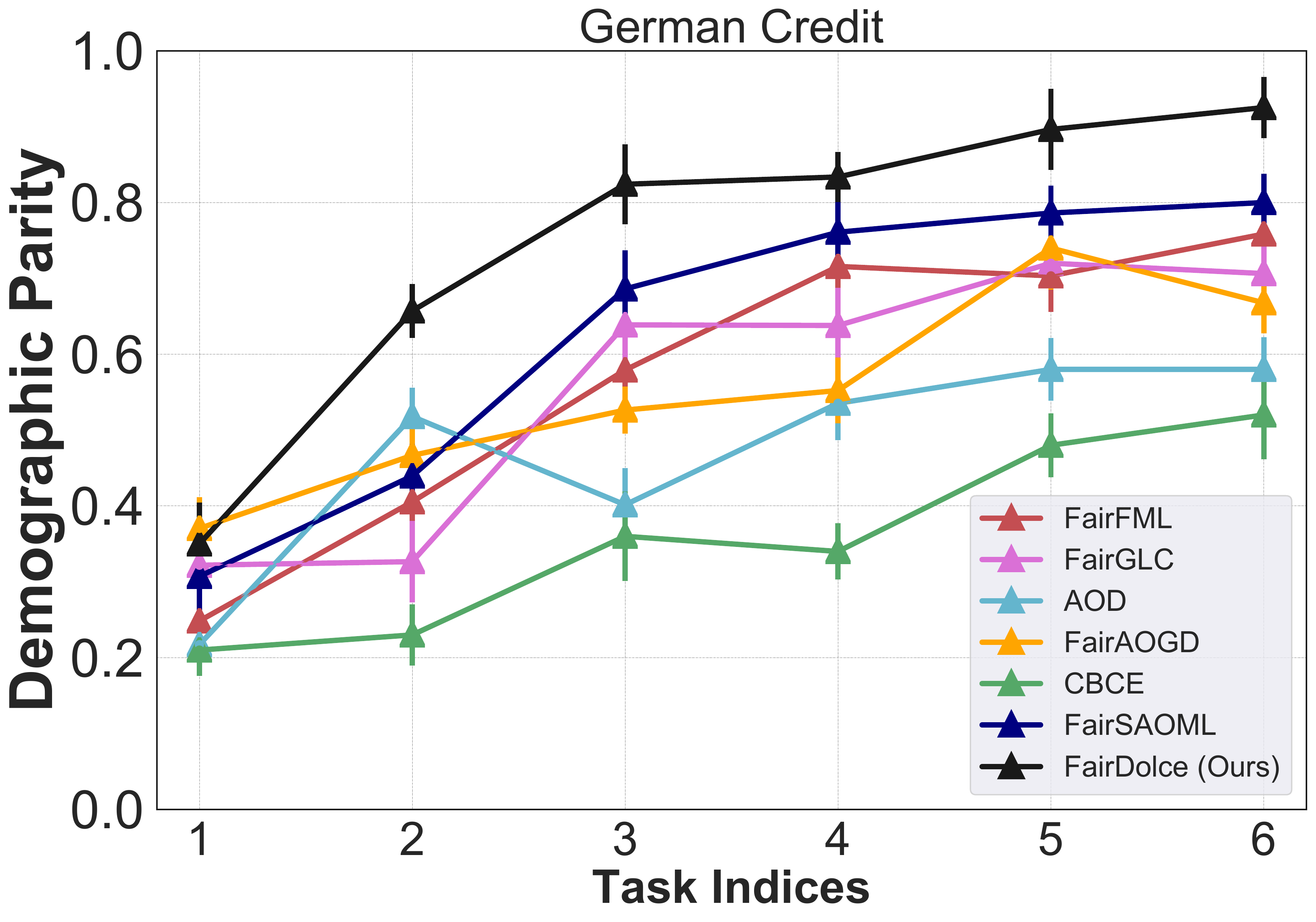}
        \caption{}
    \end{subfigure}
    \begin{subfigure}[b]{0.245\textwidth}
        \includegraphics[width=\textwidth]{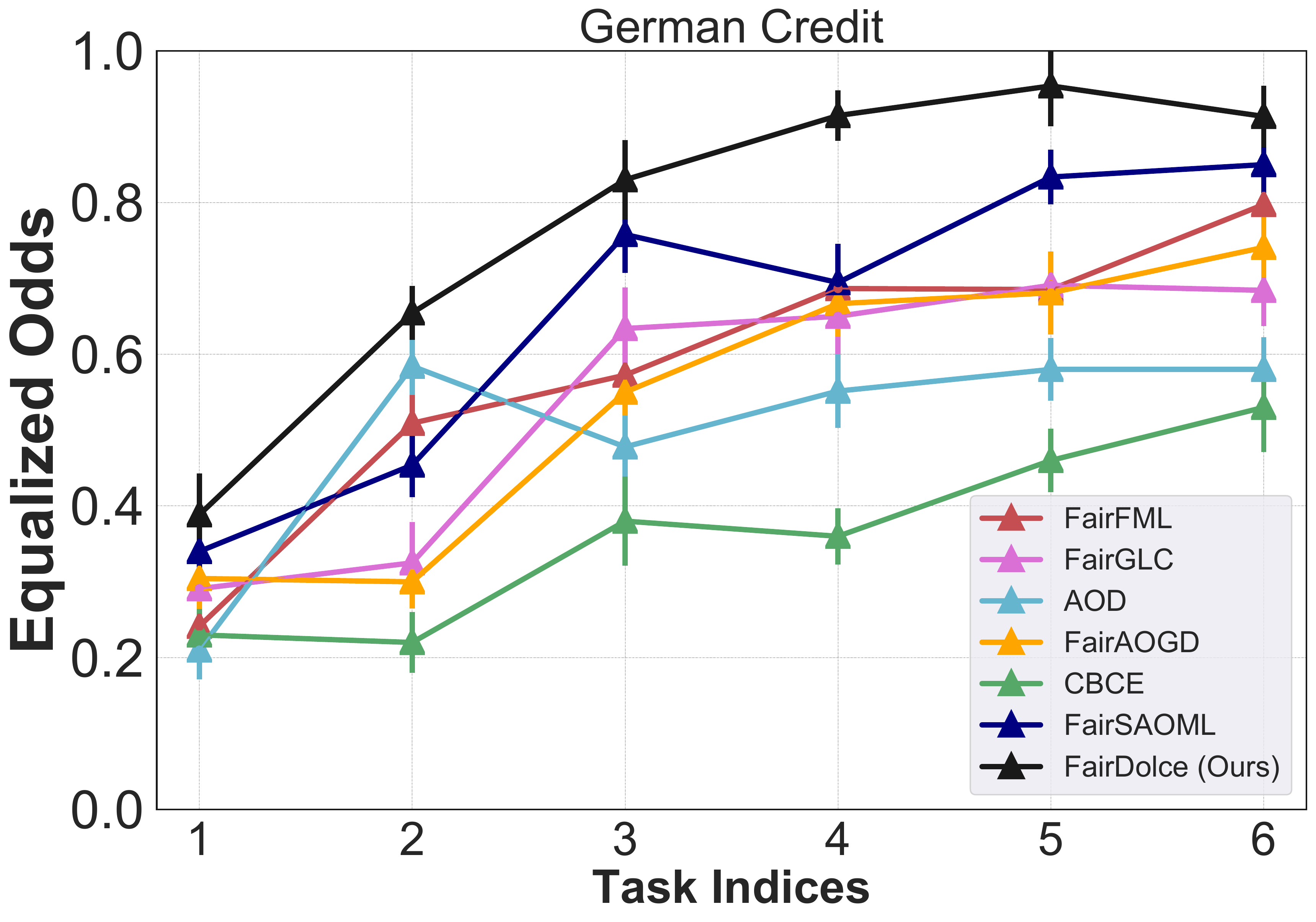}
        \caption{}
    \end{subfigure}
    \begin{subfigure}[b]{0.245\textwidth}
        \includegraphics[width=\textwidth]{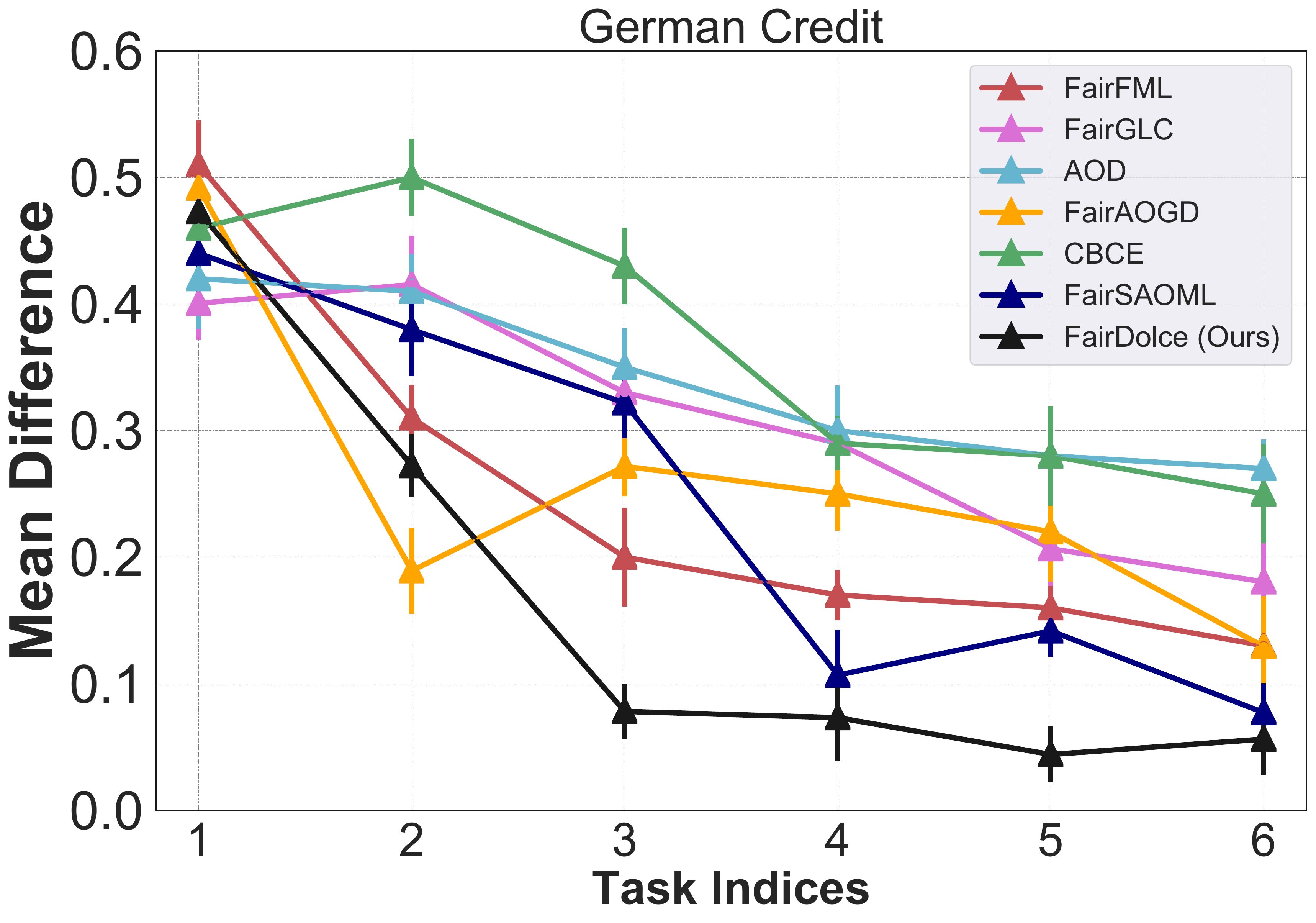}
        \caption{}
    \end{subfigure}
    \begin{subfigure}[b]{0.245\textwidth}
        \includegraphics[width=\textwidth]{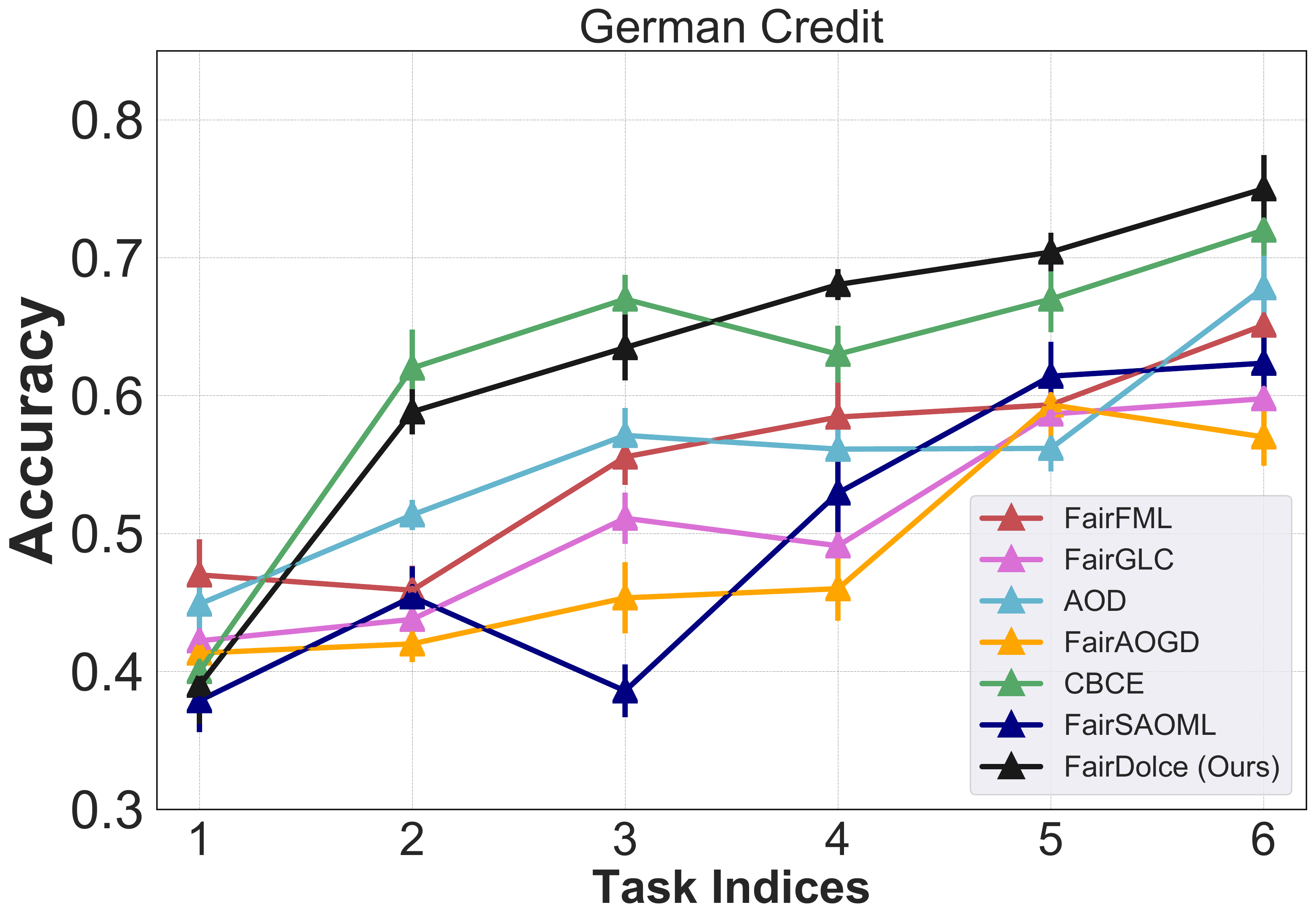}
        \caption{}
    \end{subfigure}
    \vspace{-3mm}
    \caption{Model performance over datasets through each time. (a-d) Rotated-Colored-MNIST;  (e-h) New York Stop-and-Frisk, (i-l) Chicago Crime; (m-p) German Credit.}
    \label{fig:all-data-performance}
\vspace{-3mm}
\end{figure*}

\begin{figure*}[!t]
\captionsetup[subfigure]{aboveskip=-1pt,belowskip=-1pt}
\centering
    \begin{subfigure}[b]{0.245\textwidth}
        \includegraphics[width=\textwidth]{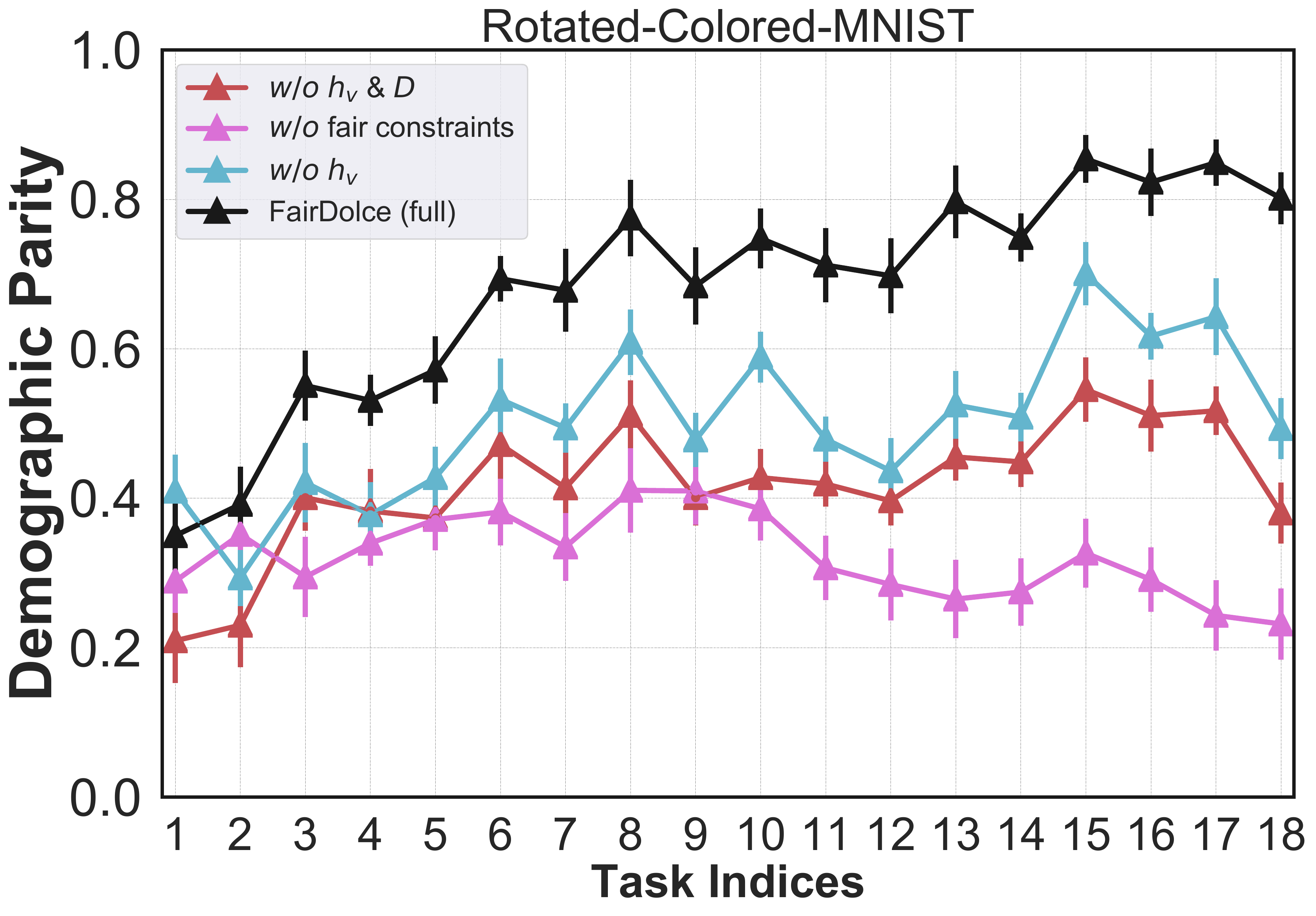}
    \end{subfigure}
    \begin{subfigure}[b]{0.245\textwidth}
        \includegraphics[width=\textwidth]{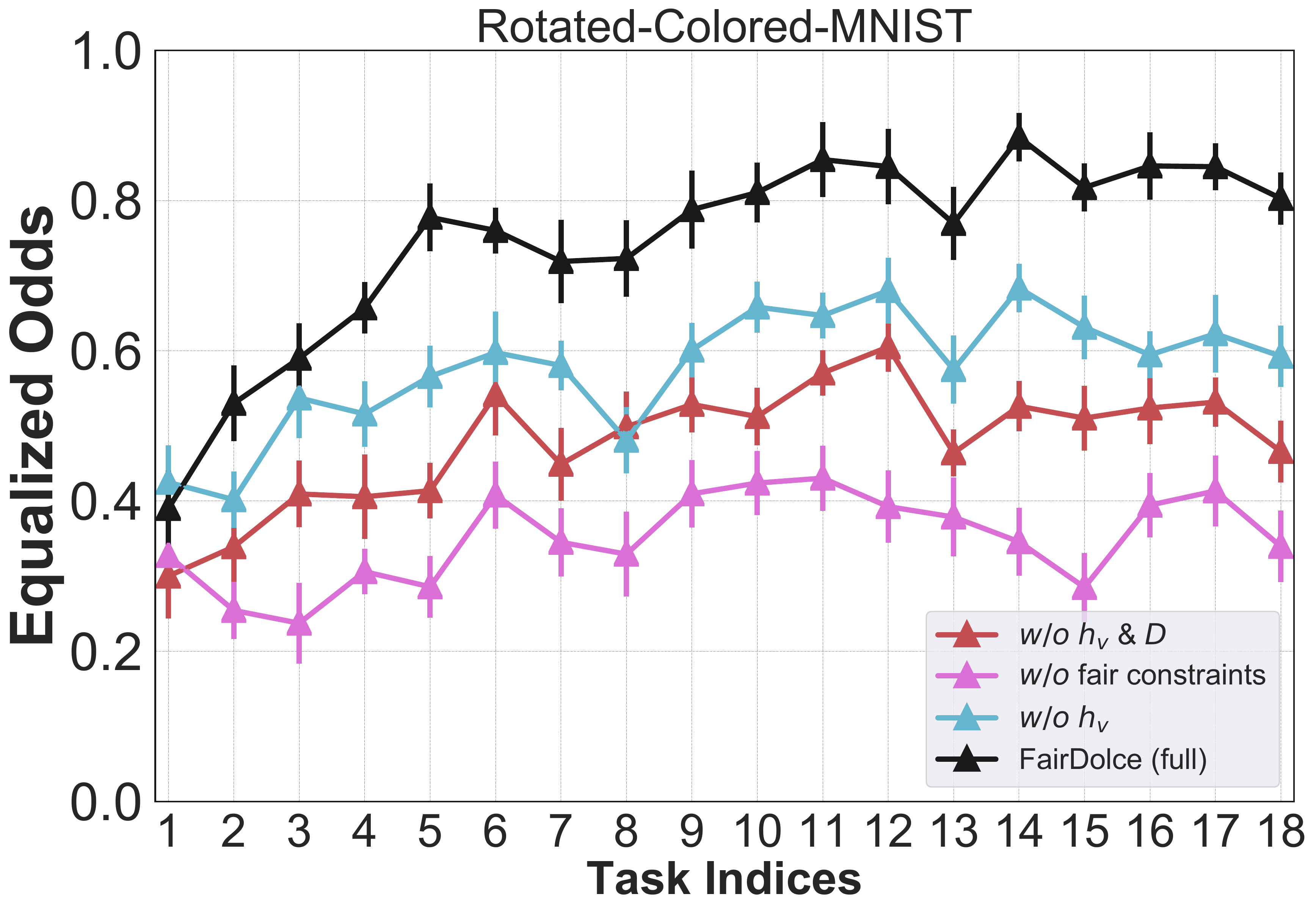}
    \end{subfigure}
    \begin{subfigure}[b]{0.245\textwidth}
        \includegraphics[width=\textwidth]{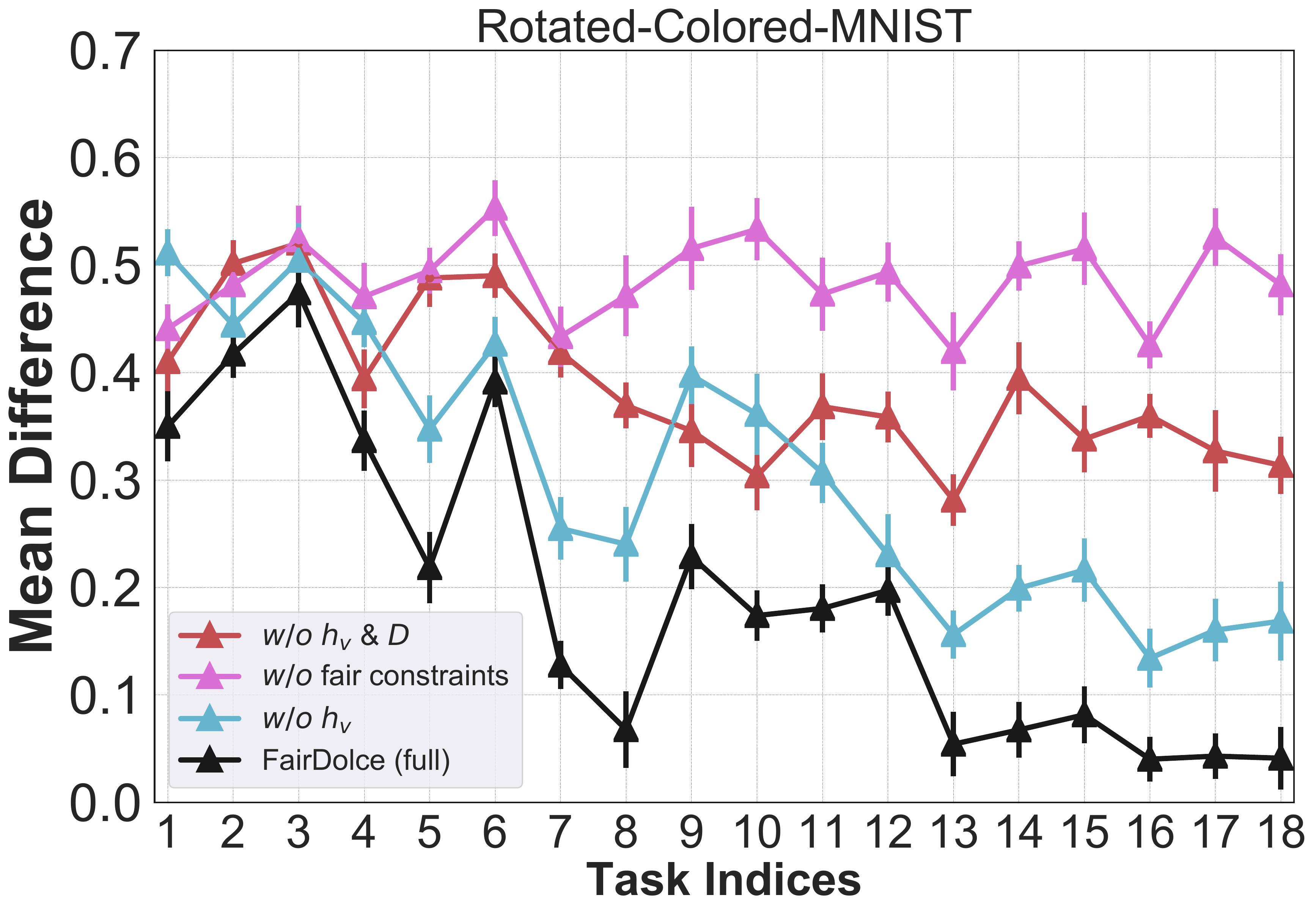}
    \end{subfigure}
    \begin{subfigure}[b]{0.245\textwidth}
        \includegraphics[width=\textwidth]{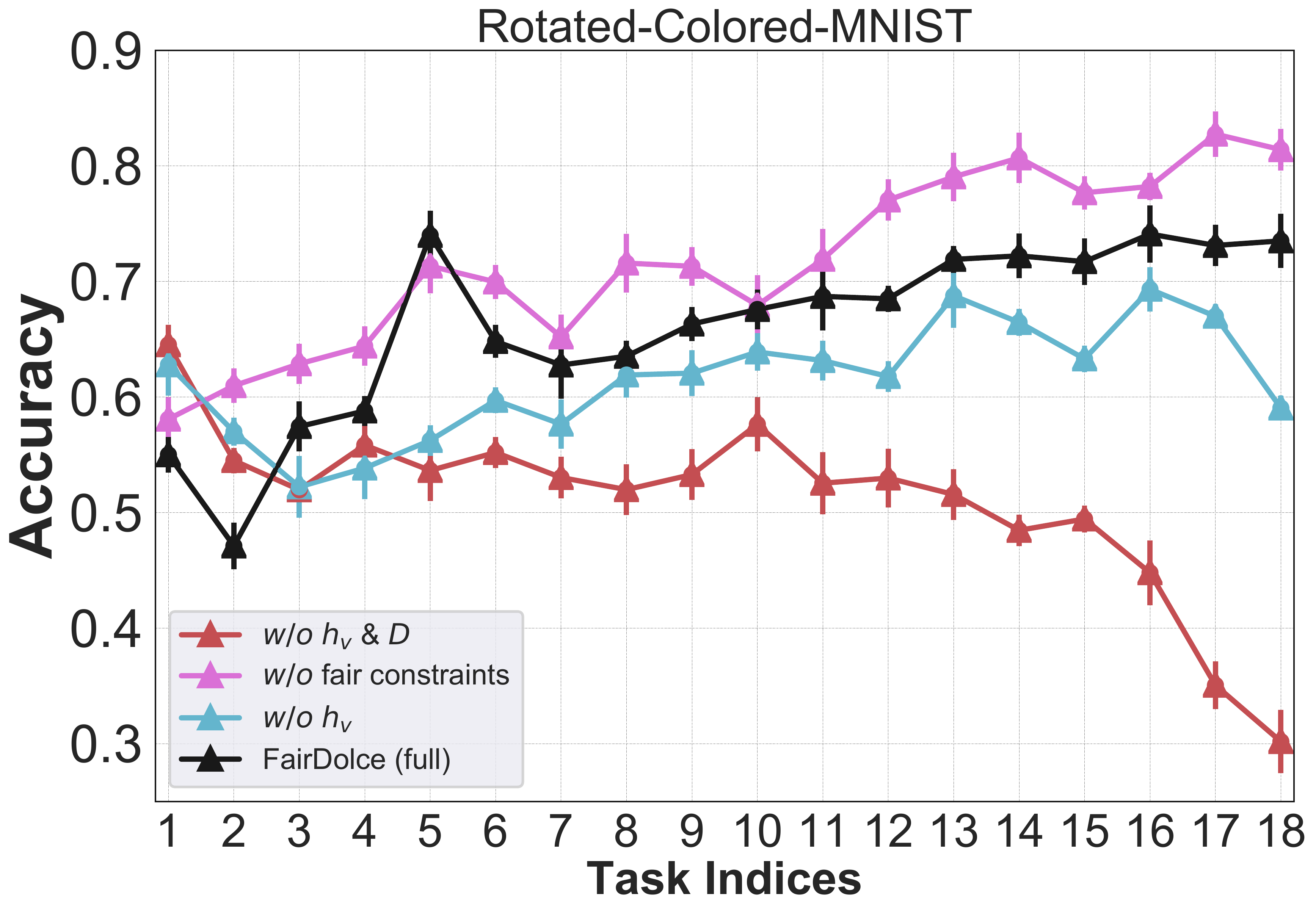}
    \end{subfigure}
    \vspace{-3mm}
    \caption{Ablation studies on the Rotated-Colored-MNIST dataset.}
    \label{fig:abs-rcmnist}
\vspace{-3mm}
\end{figure*}

\textbf{Baselines.}
We compare the performance of our proposed \sysname{} with six baseline methods from three perspectives: online learning for changing environments (AOD \cite{zhang-2020-AISTATS}, CBCE \cite{Jun-2017-AISTATS}), online fairness learning (FairFML \cite{zhao-KDD-2021}, FairAOGD \cite{AdpOLC-2016-ICML}, FairGLC \cite{GenOLC-2018-NeurIPS},), and the state-of-the-art online fairness learning for changing environments (FairSAOML \cite{zhao-KDD-2022}). AOD minimizes the strongly adaptive regret by running multiple online gradient descent algorithms over a set of dense geometric covering intervals. CBCE adapts changing environment in an online learning paradigm by combining the idea of sleeping bandits with the coin betting algorithm. FairFML controls bias in an online working paradigm and aims to attain zero-shot generalization with task-specific adaptation. FairFML focuses on a static environment and assumes tasks are sampled from an unchangeable distribution. FairAOGD is proposed for online learning with long-term constraints. In order to fit bias-prevention and compare them to \sysname{}, we specify such constraints as DDP stated in \cref{def:fairness-notion}. FairGLC rectifies FairAOGD by square-clipping the constraints in place of $g_i(\cdot), \forall i$. FairSAOML addresses fair online learning in changing environments by dynamically activating a subset of learning processes at each time through different combinations of task sets. 

\textbf{Architectures.}
For the rcMNIST image dataset, all images are resized to $28\times28$. Following \cite{zhang2022towards}, the semantic encoder and the style encoder consist of 4 strided convolutional layers followed by ReLU activation functions and Batch Normalization \cite{ioffe2015batch}. 
The decoder consists of 1 upsampling layer and 6 strided convolutional layers activated by ReLU. The classifier is performed by 2 FC hidden layers activated by ReLU. For tabular datasets (\textit{i.e.,} New York Stop-and-Frisk, Chicago Crime, and German Credit), followed by \cite{oh2022learning}, both encoders and the decoder contain one FC layer followed by LeakyReLU activation functions. The network architecture for the classifier is 1 FC layer activated by Sigmoid. Details for hyperparameters tuning are provided in \cref{sec:app-hyper-search}.

%% file: results.tex
\subsection{Adaptability for Changing Environments}
As shown in the first three columns in \cref{fig:all-data-performance}, model performance is sequentially evaluated by fairness metrics (\textit{i.e.,} DP, EO, and MD) introduced in \cref{sec:experiments}. Our results demonstrate \sysname{} outperforms baseline methods by giving the highest DP and EO values and the lowest MD overall. Specifically, it eventually meets the fair criteria of "80\%-rule" \cite{Biddle-Gower-2005} where DP and EO at the last several times are beyond 0.8. The last column of \cref{fig:all-data-performance} shows the change of model accuracies over time. We claim that \sysname{} substantially outperforms alternative approaches with robust performance under dynamic environments in achieving the highest accuracy of all time. 

As a tough competitor, FairSAOML addresses the same problem that we stated in this paper by proposing an expert-tracking technique in which experts' weights are updated accordingly. It assumes that larger experts containing information across a large number of tasks help the learner to adapt to the new environment quickly. In our experiments, although FairSAOML shows competitive performance in bias control, it cannot surpass ours. This is because when the environment changes, larger experts in FairSAOML retain information from the old environments, which hurts the performance of the learner. Similar reasons are attributed to interval-based learning algorithms, such as AOD and CBCE. In the case of changing environments, one major merit of \sysname{} is to disentangle data by separated representations in latent spaces, where only the semantic ones correspond to model predictions. This effectively controls the interference from various environments.

\vspace{-3mm}

\subsection{Ablation Studies}
We conduct ablation studies on all datasets  to demonstrate the contributions of three key components in our method. \cref{fig:abs-rcmnist} demonstrates the results on the rcMNIST dataset. Results on other datasets refer to \cref{fig:abs-ny,fig:abs-chicago,fig:abs-german} in \cref{app:exp-results}. 
(1) In this first study ($w/o\: h_v\:\&\:D$), we intentionally remove the variation encoder $h_v$ and the decoder $D$ and only keep the semantic encoder $h_s$ and the classifier $\omega$ with fair constraints. In this sense, the proposed architecture is equivalent to a simple neural network, and the semantic encoder functions as a featurizer. 
(2) In the second study ($w/o\: \text{fair constraints}$), we keep all modules but remove the fairness constraints $g$ from the classifier $\omega$. Without fair constraints, although the model provides better performance, fairness is not guaranteed over time. 
(3) In the third study ($w/o\: h_v$), only the variation encoder $h_v$ is removed. Without the variation encoder, the model is similar to conventional auto-encoders. The generalization ability to changing environments is weakened.

%% file: conclusion.tex
To address the problem of fairness-aware online learning for changing environments, we first introduce a novel regret, namely \sysnameregret{}, in which it takes a mixed form of static and dynamic regret metrics. We challenge existing online learning methods by sequentially updating model parameters with a local change, where only parts of the parameters correspond to environmental change, and keep the remaining invariant to environments and thus solely for fair predictions. To this end, an effective algorithm \sysname{} is introduced, wherein it consists of two networks with auto-encoders. Through disentanglement, data are able to be encoded with an environment-invariant semantic factor and an environment-specific variation factor. Furthermore, semantic factors are used to learn a classifier under a group fairness constraint.
Detailed theoretic analysis and corresponding proofs justify the effectiveness of the proposed algorithm by demonstrating upper bounds for the loss regret and violation of fair constraints. Empirical studies based on real-world datasets show that our method outperforms state-of-the-art online learning techniques in both model accuracy and fairness.

%% file: app_expdetails.tex
\subsection{Notations}
\label{sec:notations}

Vectors are denoted by lowercase bold face letters. Scalars are denoted by lowercase italic letters. Sets are denoted by uppercase calligraphic letters. Indices of task sequences are denoted as $[T]=\{1,\cdots,T\}$. $||\cdot||$ represents $\ell_2$ norm.

\begin{table}[!h]
    \centering
    \caption{Important notations and corresponding descriptions.}
    \vspace{-3mm}
    \begin{tabular}{c|p{6cm}}
        \toprule
        \textbf{Notations} & \textbf{Descriptions}  \\
        \cmidrule(lr){1-2}
        $T$ & total number of learning tasks\\
        $t$ & indices of tasks\\
        $f_t$ & loss function at time $t$  \\
        $g$ & fairness function \\
        $\omega$ & classification function\\
        $h_s$ & semantic encoder\\
        $h_v$ & variation encoder\\
        $D$ & decoder\\
        
        $\boldsymbol{\theta}$ & model parameters\\
        $\boldsymbol{\theta}^s$ & parameters of the semantic encoder\\
        $\boldsymbol{\theta}^v$ & parameters of the variation encoder\\
        $\boldsymbol{\theta}^d$ & parameters of the decoder\\
        $\boldsymbol{\theta}^{cls}$ & parameters of the classifier\\
        $\mathbf{u}^s$ & semantic comparators\\

        $\mathbf{s}$ & semantic factor (representation)\\
        $\mathbf{v}$ & variation factor \\

        $\mathcal{Q}_t$ & data batch sampled from the task pool at time $t$\\
        $Q$ & total number of quartet/doublet pairs in $\mathcal{Q}_t$\\
        $q$ & indices of quartet/doublet pair in $\mathcal{Q}_t$ \\
        $|\mathcal{Q}_t|$ & total number of samples in the batch $\mathcal{Q}_t$\\
        
        \bottomrule
    \end{tabular}
    \label{tab:notation}
\end{table}

\subsection{Hyperparameter Search}
\label{sec:app-hyper-search}
For each dataset, we tune the following hyperparameters: (1) the initial dual meta parameter $\lambda_1, \lambda_2, \lambda_3$ is chosen from $\{$0.00001, 0.0001, 0.001, 0.01, 0.1, 1, 10, 100, 1000, 10000 $\}$; (2) learning rates $\eta_1$ and $\eta_2$ for updating primal and dual variables are chosen from $\{$0.0001, 0.0005, 0.001, 0.005, 0.01, 0.05, 0.1, 0.5, 1, 5, 10, 50, 100, 500, 1000$\}$; (3) set margins $\eta_1=\eta_2=\eta_3=0.05$.

%% file: app_expresults.tex
Ablation study results on the New York Stop-and-Frisk, Chicago Crime, and German Credit datasets are shown in \cref{fig:abs-ny,fig:abs-chicago,fig:abs-german}. Similar trends are observed as the Rotated-Colored-MNIST in \cref{fig:abs-rcmnist}.

%% file: app_proof.tex
\begin{proof}
Using $\eta_{1,t} = \frac{\eta_{1,0}}{\sqrt{T}}, \eta_{2,t}=\frac{\eta_{2,0}}{\sqrt{\eta_{1,t}}}, \forall t\in[T]$ given in the Theorem yelds
\begin{align*}
    \sum_{t=1}^T \Delta_t(\mathbf{u}_T^c)&=\frac{\sqrt{T}}{\eta_{1,0}}\sum_{t=1}^T(||\mathbf{u}_T^c-\boldsymbol{\theta}_t^l||^2-||\mathbf{u}_T^c-\boldsymbol{\theta}_{t+1}^l||^2)\\
    &\leq \frac{\sqrt{T}}{\eta_{1,0}}||\mathbf{u}_T^c-\boldsymbol{\theta}_1^l||^2
\end{align*}
Combing the above inequality with the Lemma 1 presented in \cite{yi2021regret} yields
\begin{align*}
    &\sum_{t=1}^T f_t(h_s(\boldsymbol{\theta}_t^s),\boldsymbol{\theta}^{cls}_t) -\min_{\boldsymbol{\theta}^{cls}\in\Theta} \sum_{t=1}^T f_t(h_s(\mathbf{u}_t^s),\boldsymbol{\theta}^{cls}) \\
    &\leq \frac{\sqrt{T}}{\eta_{1,0}}||\mathbf{u}_T^c-\boldsymbol{\theta}_1^l||^2 + \frac{G^2\eta_{1,0}}{2}\sqrt{T}
\end{align*}
which yields the bound for the loss regret. Similarly, with the Lemma 1 presented in \cite{yi2021regret}, it yields
\begin{align*}
    ||\lambda_{T,1}||^2\leq\beta T
\end{align*}
where $\beta=\frac{2}{\eta_{1,0}\sqrt{T}}||\mathbf{u}_T^c-\boldsymbol{\theta}_1^l||^2+\frac{G^2\eta_{1,0}}{\sqrt{T}}+\frac{\eta_{2,0}^2 F^2}{\eta_{1,0}T T}+2F$.
Together the above inequality with $\eta_{1,t} = \frac{\eta_{1,0}}{\sqrt{T}}, \eta_{2,t}=\frac{\eta_{2,0}}{\sqrt{\eta_{1,t}}}, \forall t\in[T]$, we have
\begin{align*}
    \sum_{t=1}^T \Big|\Big|\big[g(h_s(\boldsymbol{\theta}_t^s),\boldsymbol{\theta}^{cls}_t)\big]_+\Big|\Big| 
    \leq \frac{\sqrt{m\eta_{1,0}}}{\eta_{2,0}T}||\lambda_{{T,1}}||
    \leq \frac{m\eta_{1,0}\beta}{\eta_{2,0}}T^{\frac{1}{4}}
\end{align*}
which yields the bounds for the violation of the long-term constraints.
\end{proof}

\begin{figure*}[!h]
\captionsetup[subfigure]{aboveskip=-1pt,belowskip=-1pt}
\centering
    \begin{subfigure}[b]{0.245\textwidth}
        \includegraphics[width=\textwidth]{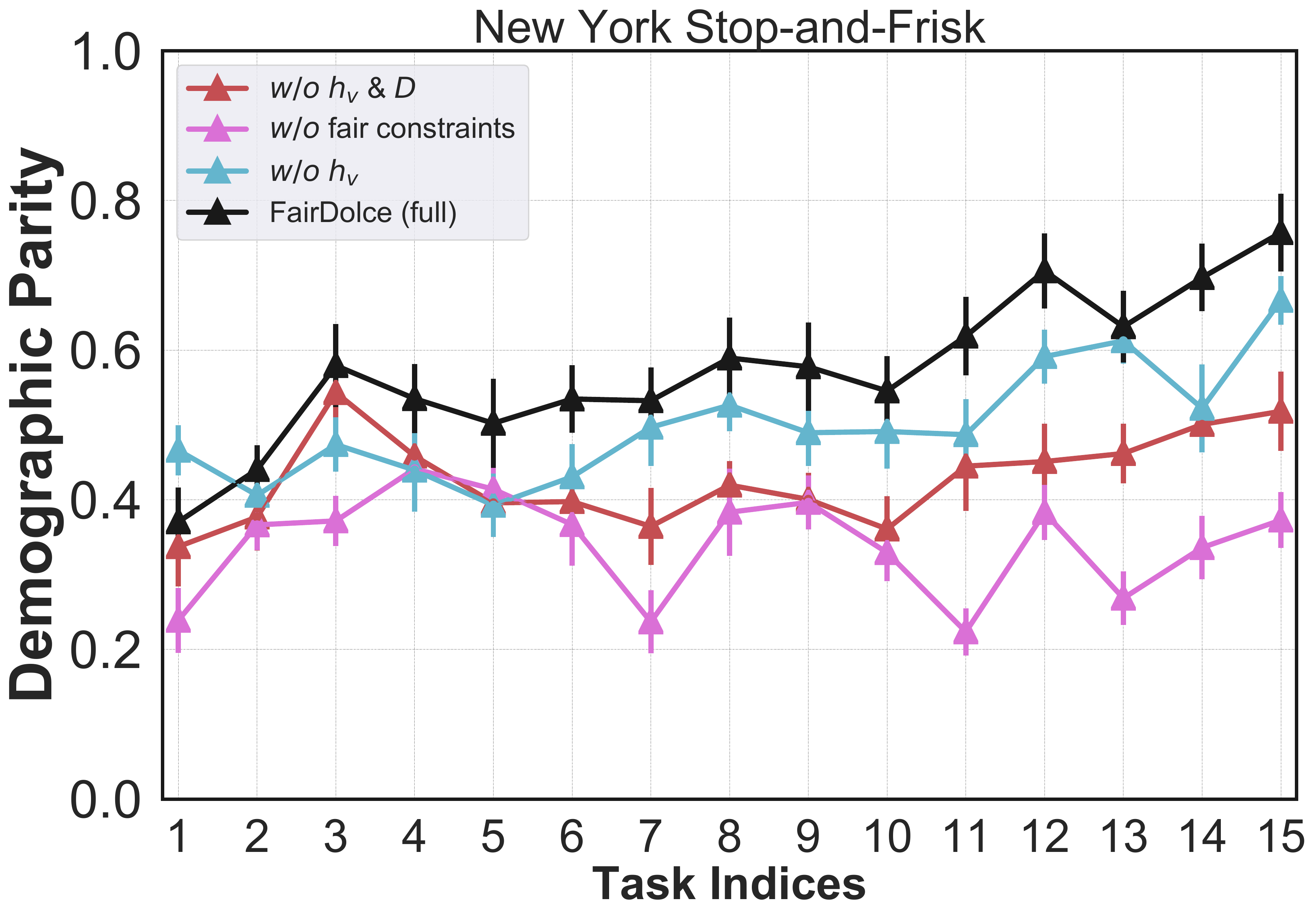}
    \end{subfigure}
    \begin{subfigure}[b]{0.245\textwidth}
        \includegraphics[width=\textwidth]{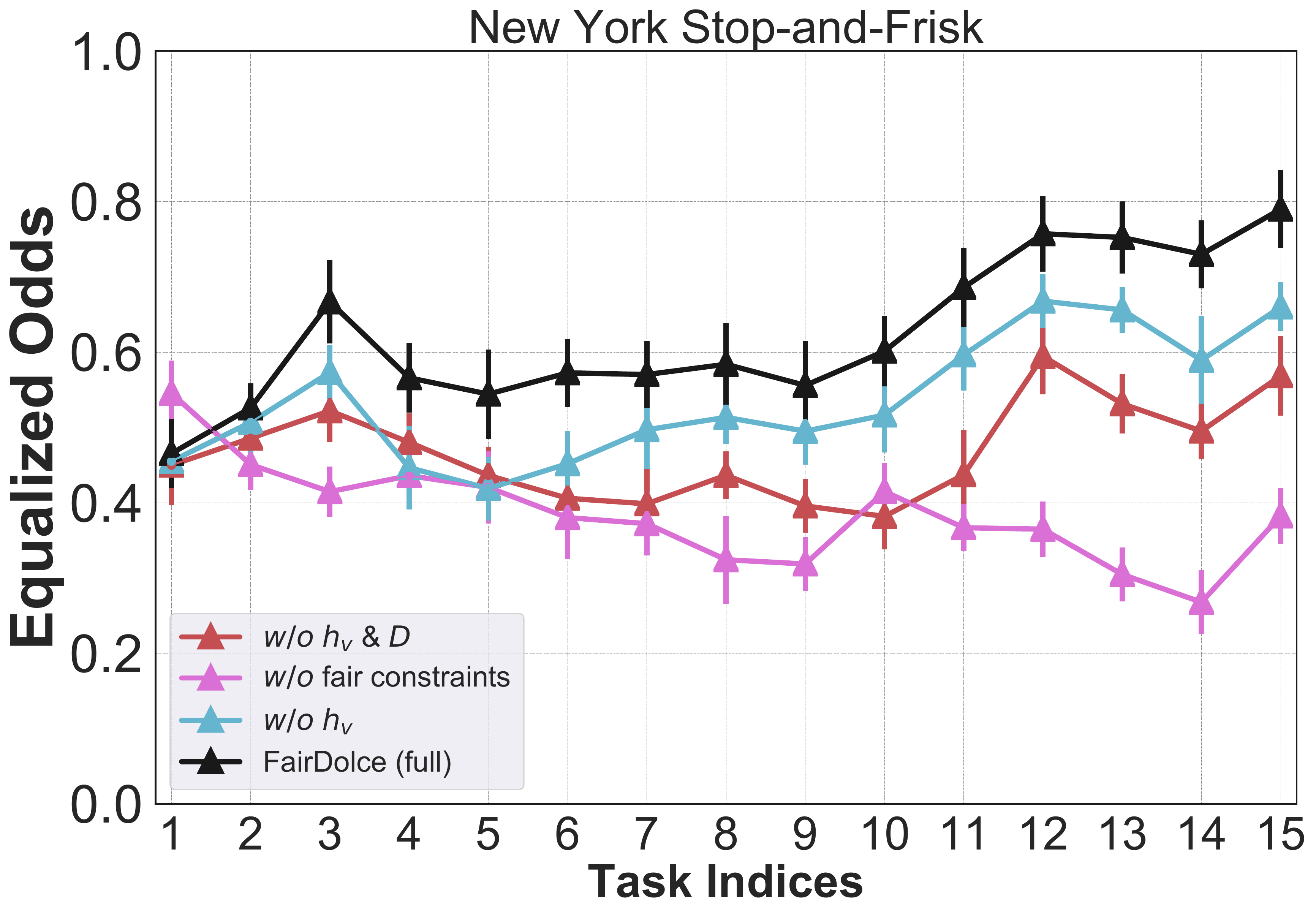}
    \end{subfigure}
    \begin{subfigure}[b]{0.245\textwidth}
        \includegraphics[width=\textwidth]{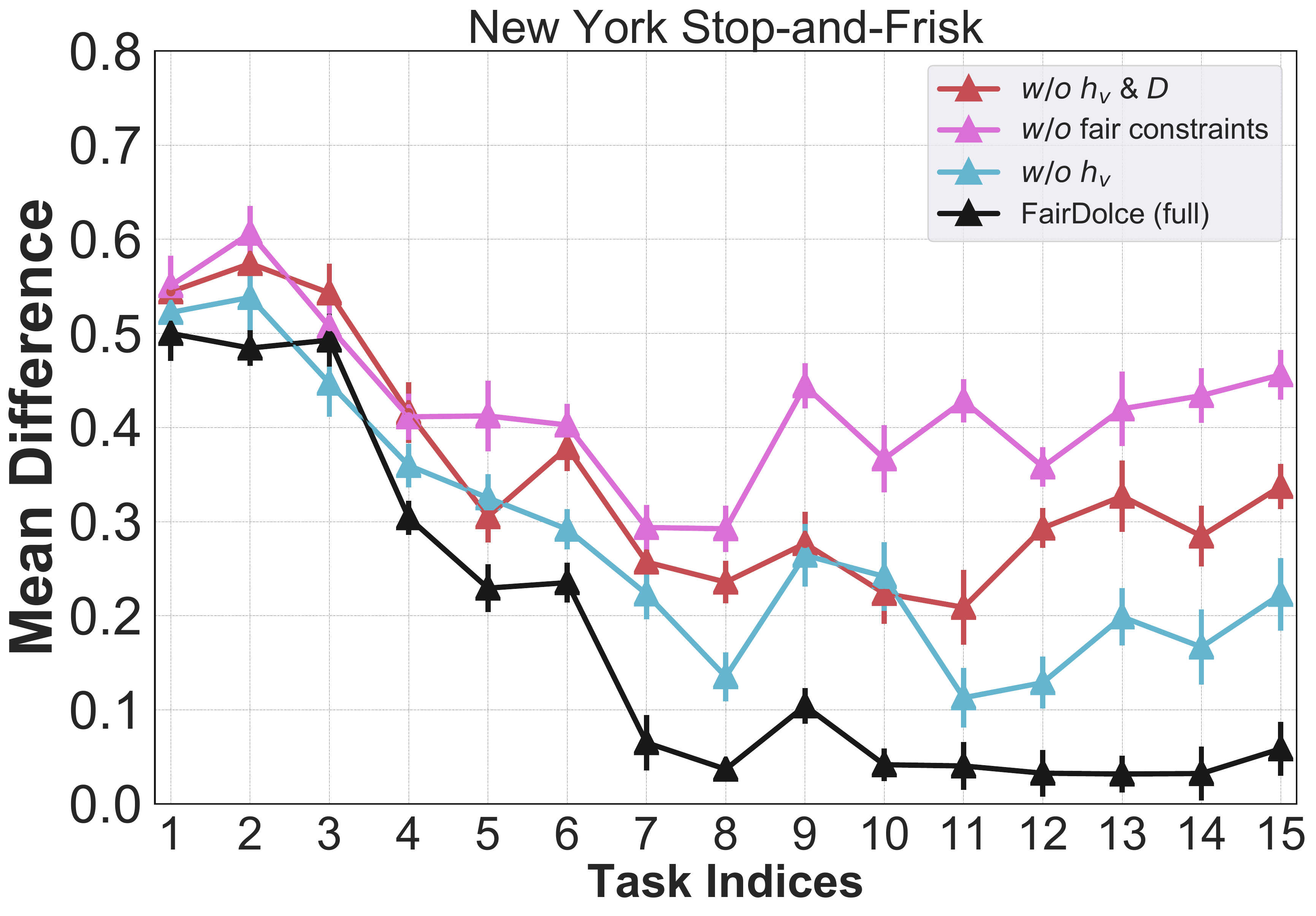}
    \end{subfigure}
    \begin{subfigure}[b]{0.245\textwidth}
        \includegraphics[width=\textwidth]{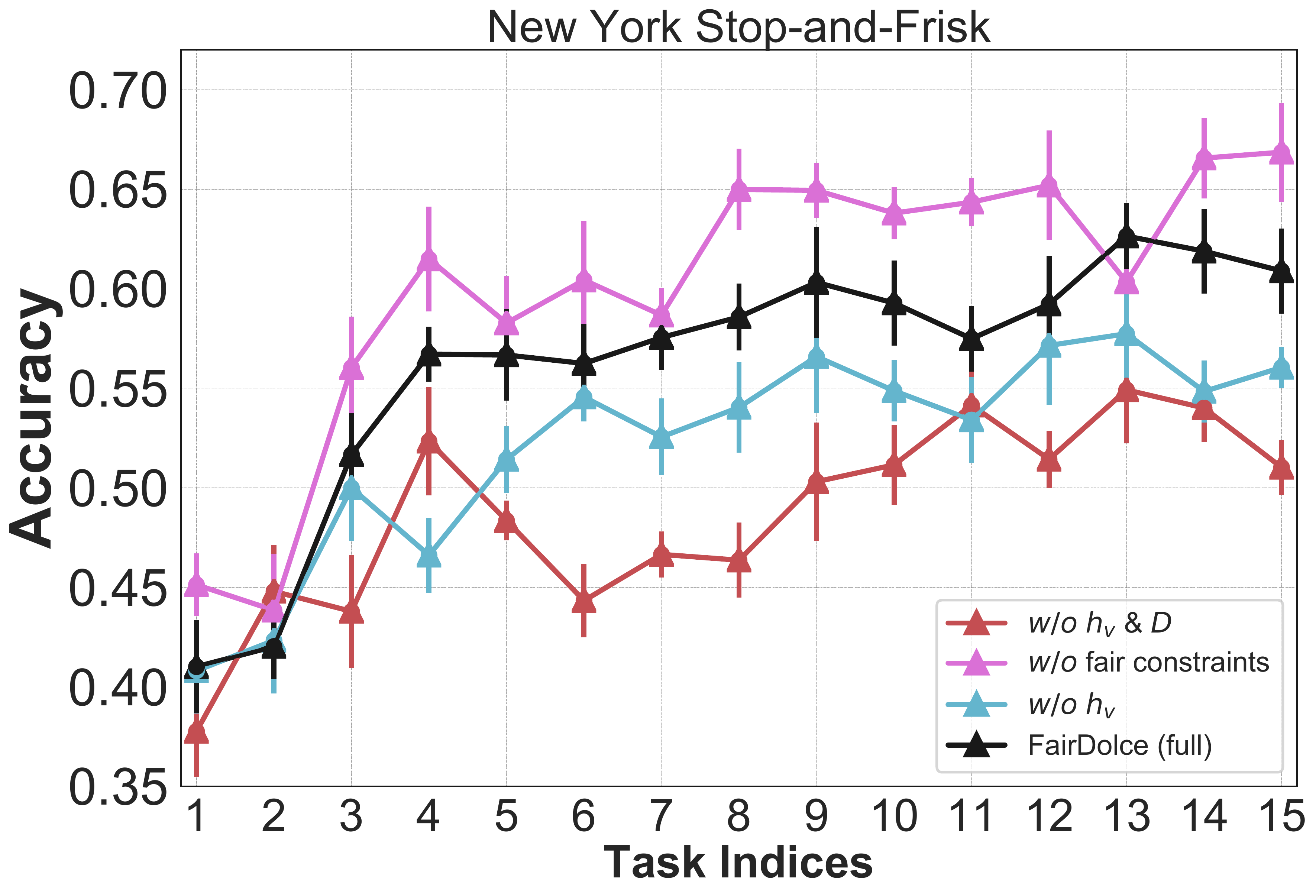}
    \end{subfigure}
    \caption{Ablation studies on the New York Stop-and-Frisk dataset.}
    \label{fig:abs-ny}
\end{figure*}

\begin{figure*}[!h]
\captionsetup[subfigure]{aboveskip=-1pt,belowskip=-1pt}
\centering
    \begin{subfigure}[b]{0.245\textwidth}
        \includegraphics[width=\textwidth]{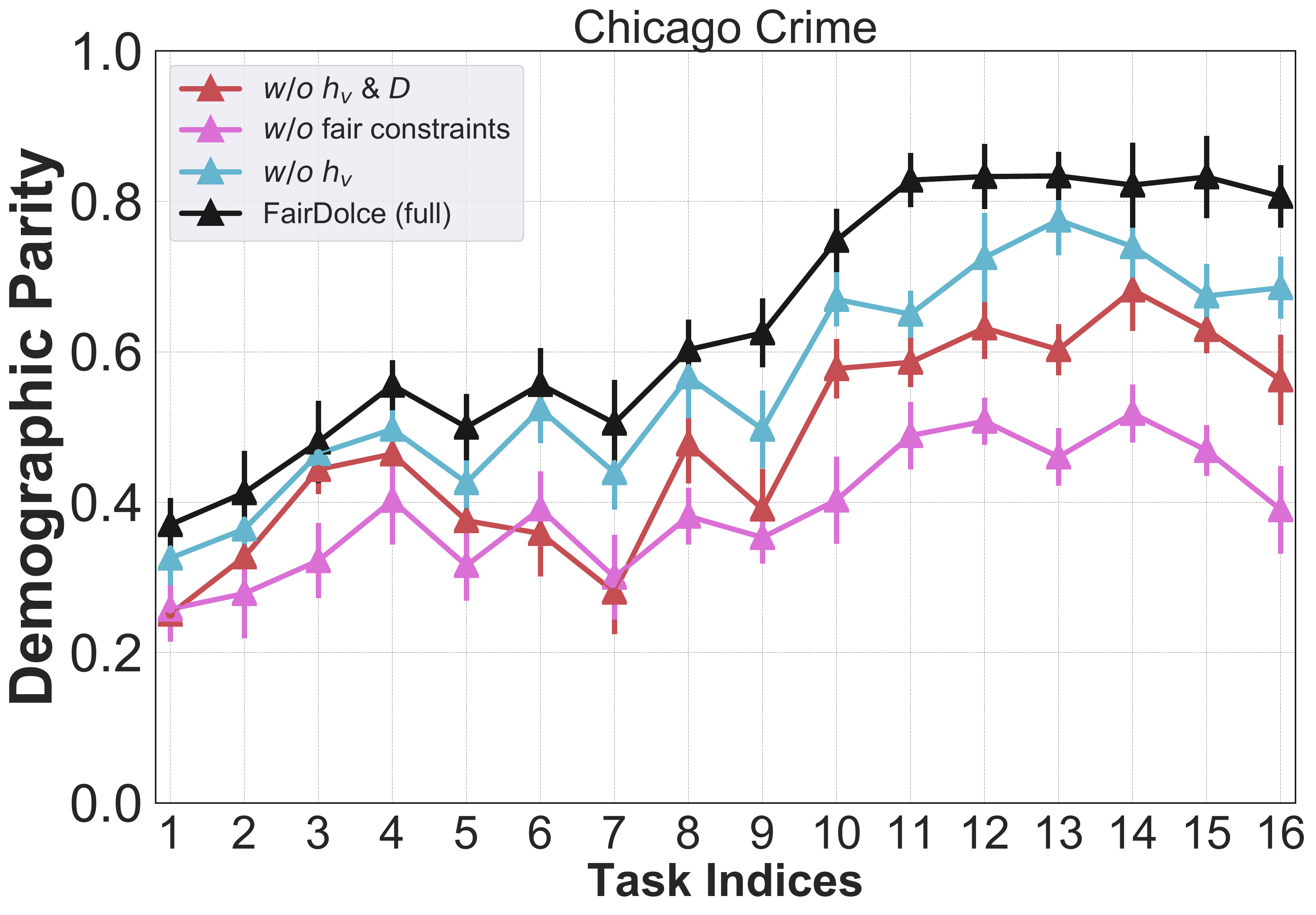}
    \end{subfigure}
    \begin{subfigure}[b]{0.245\textwidth}
        \includegraphics[width=\textwidth]{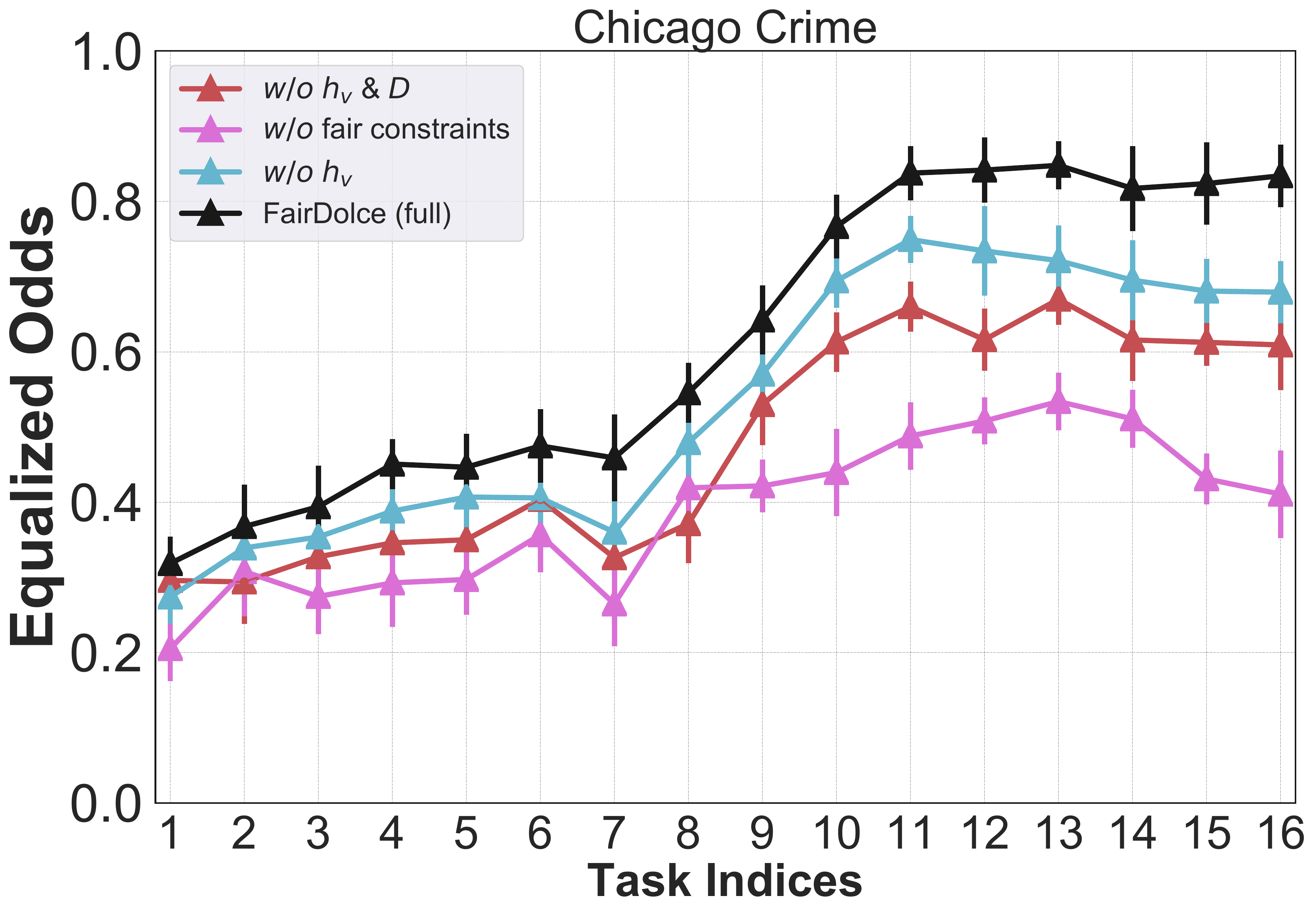}
    \end{subfigure}
    \begin{subfigure}[b]{0.245\textwidth}
        \includegraphics[width=\textwidth]{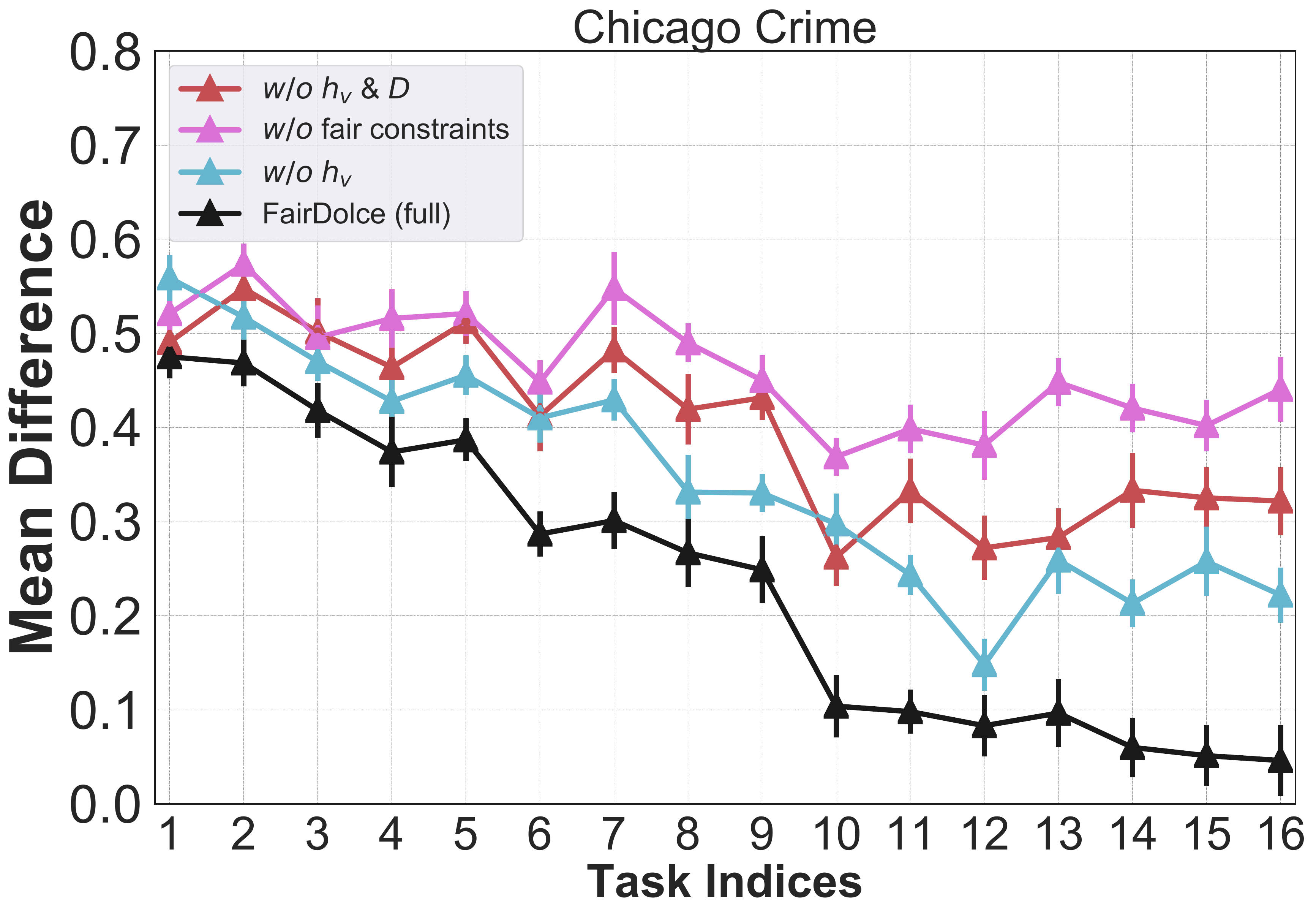}
    \end{subfigure}
    \begin{subfigure}[b]{0.245\textwidth}
        \includegraphics[width=\textwidth]{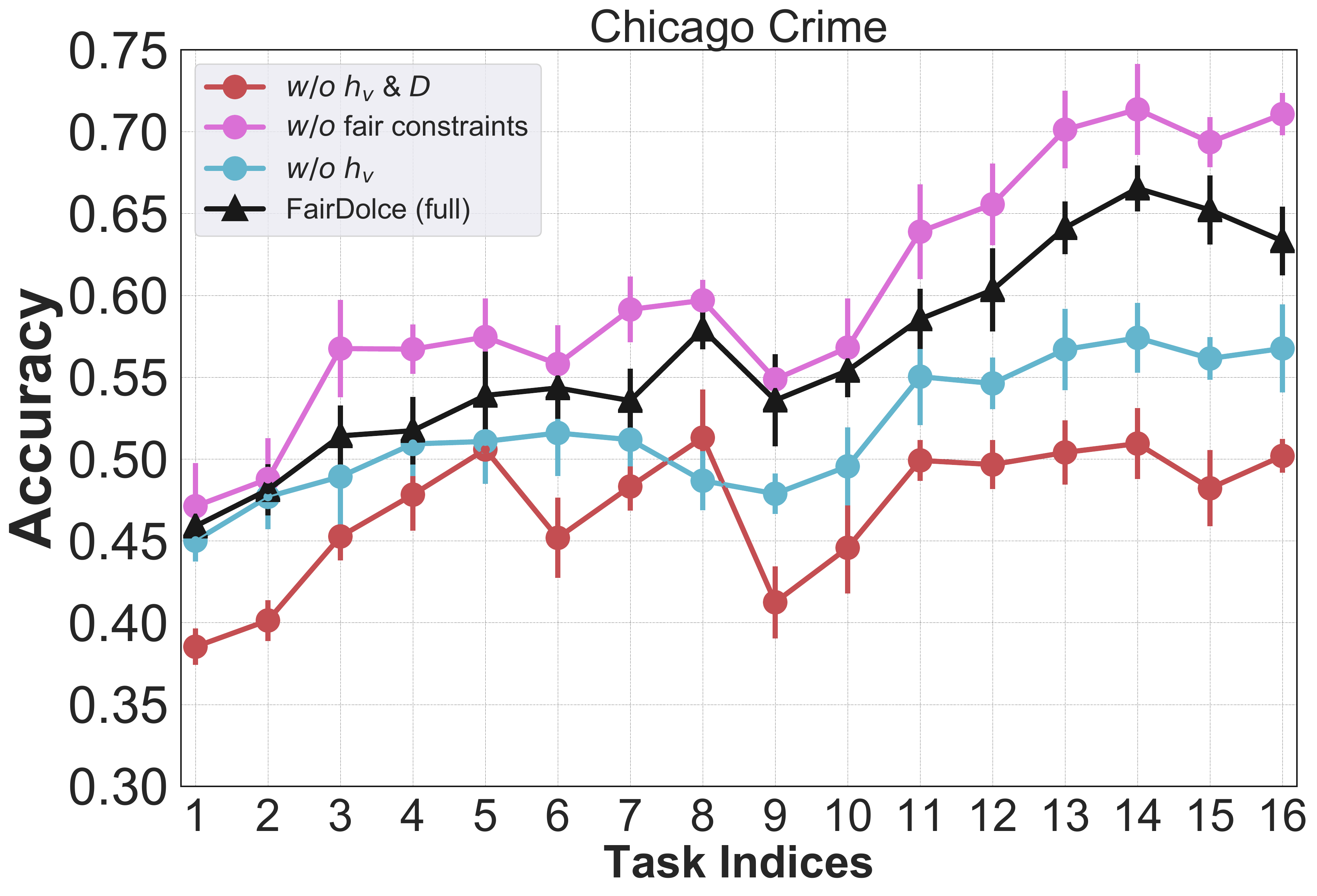}
    \end{subfigure}
    \caption{Ablation studies on the Chicago Crime dataset.}
    \label{fig:abs-chicago}
\end{figure*}

\begin{figure*}[!h]
\captionsetup[subfigure]{aboveskip=-1pt,belowskip=-1pt}
\centering
    \begin{subfigure}[b]{0.245\textwidth}
        \includegraphics[width=\textwidth]{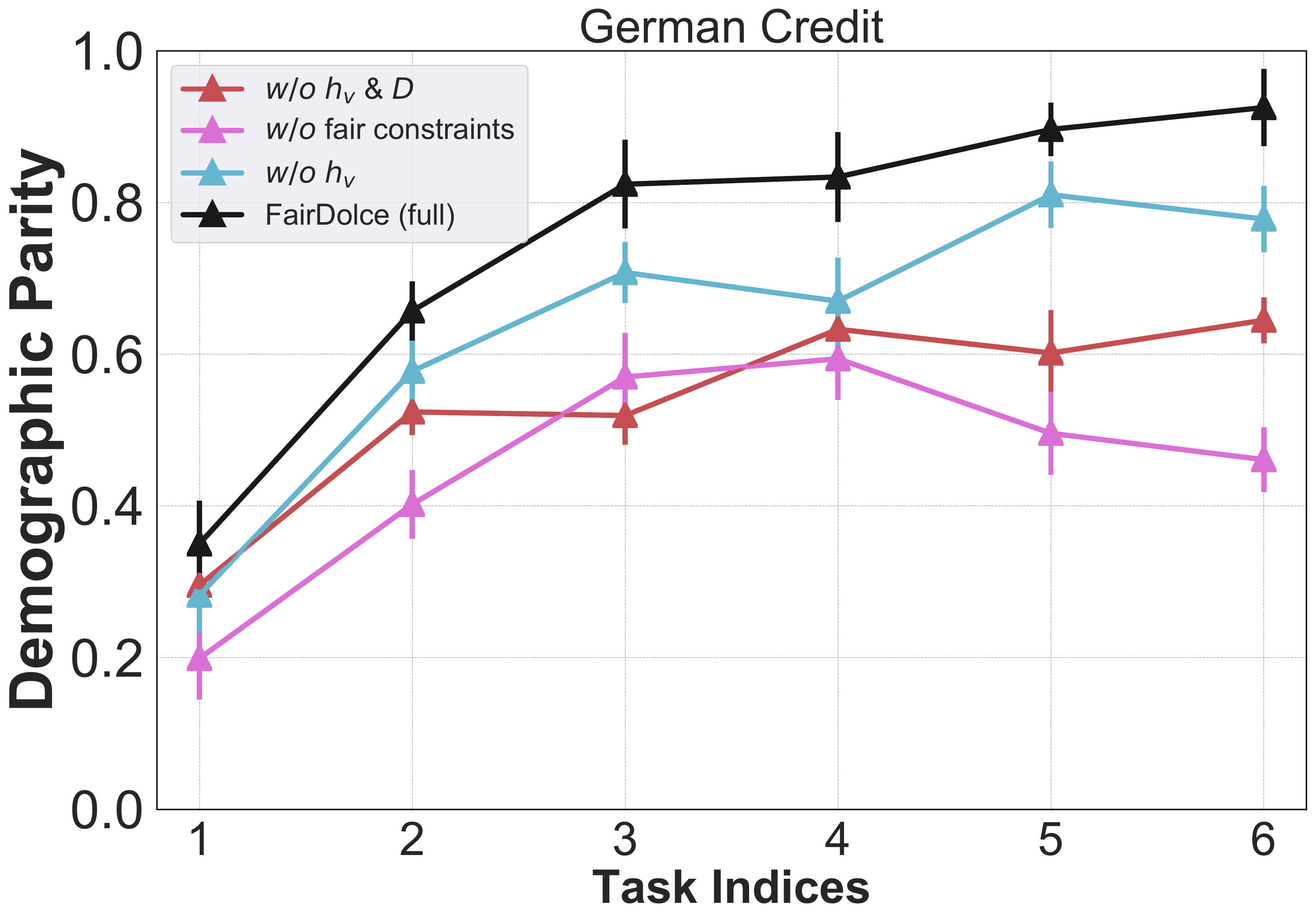}
    \end{subfigure}
    \begin{subfigure}[b]{0.245\textwidth}
        \includegraphics[width=\textwidth]{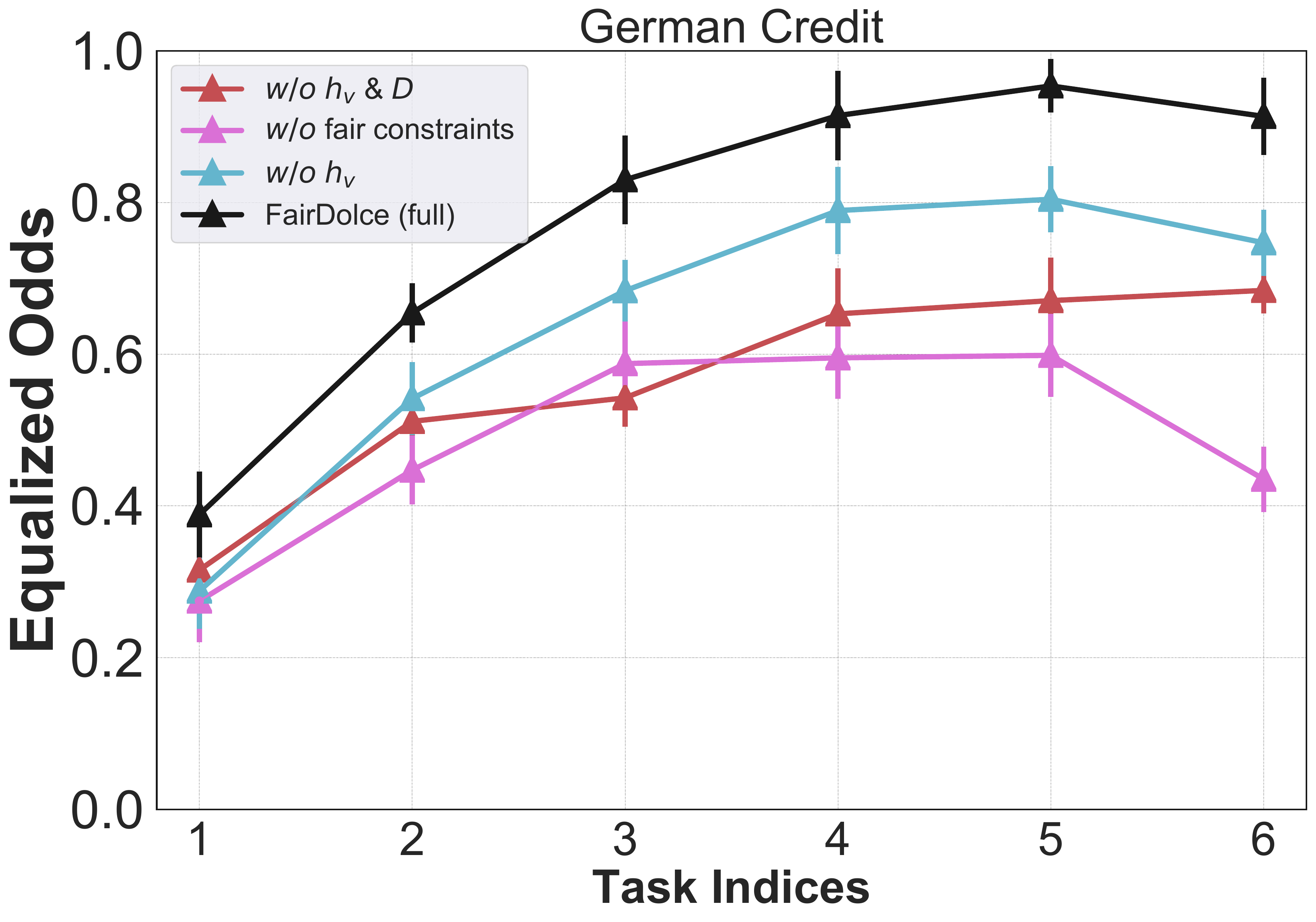}
    \end{subfigure}
    \begin{subfigure}[b]{0.245\textwidth}
        \includegraphics[width=\textwidth]{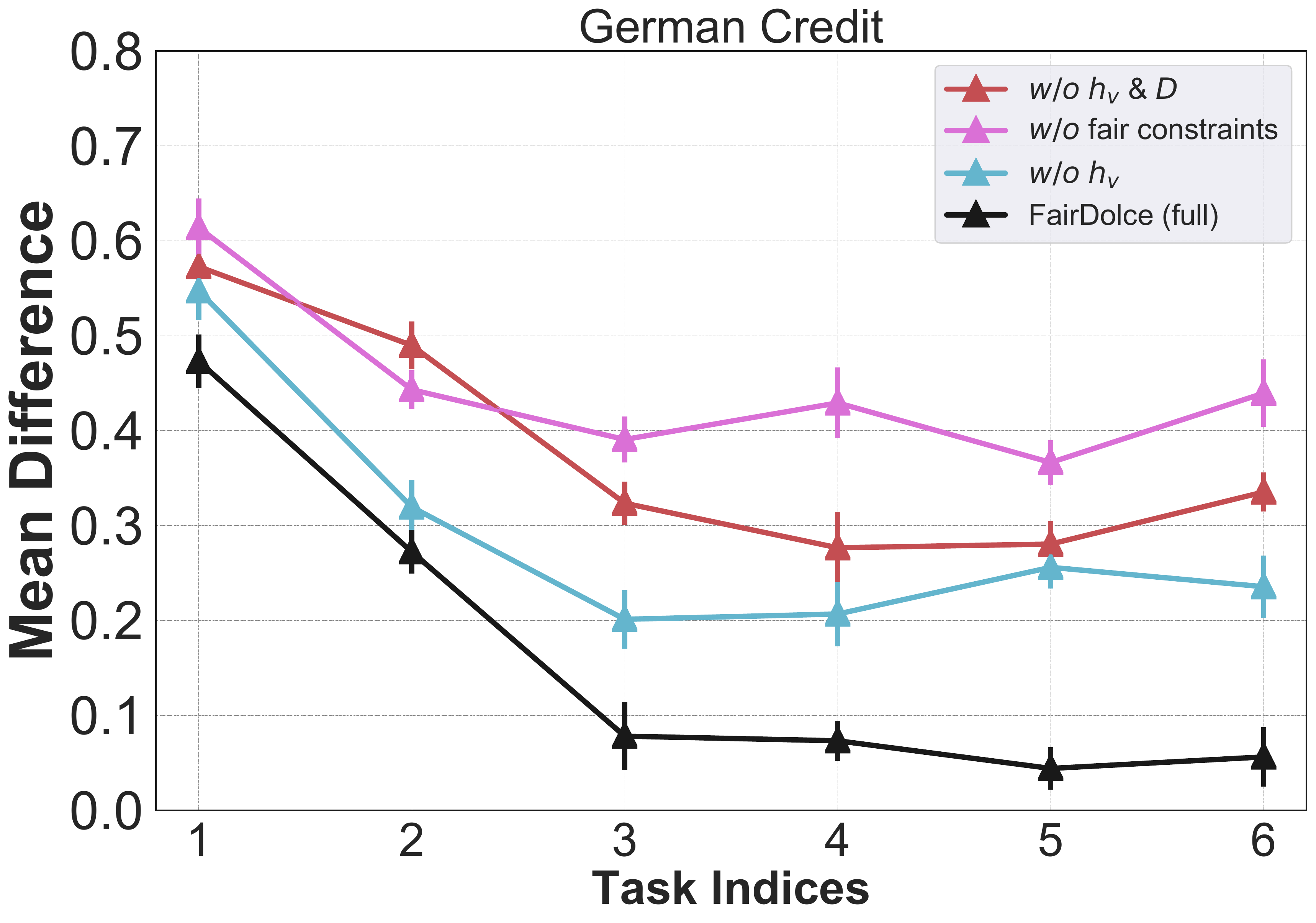}
    \end{subfigure}
    \begin{subfigure}[b]{0.245\textwidth}
        \includegraphics[width=\textwidth]{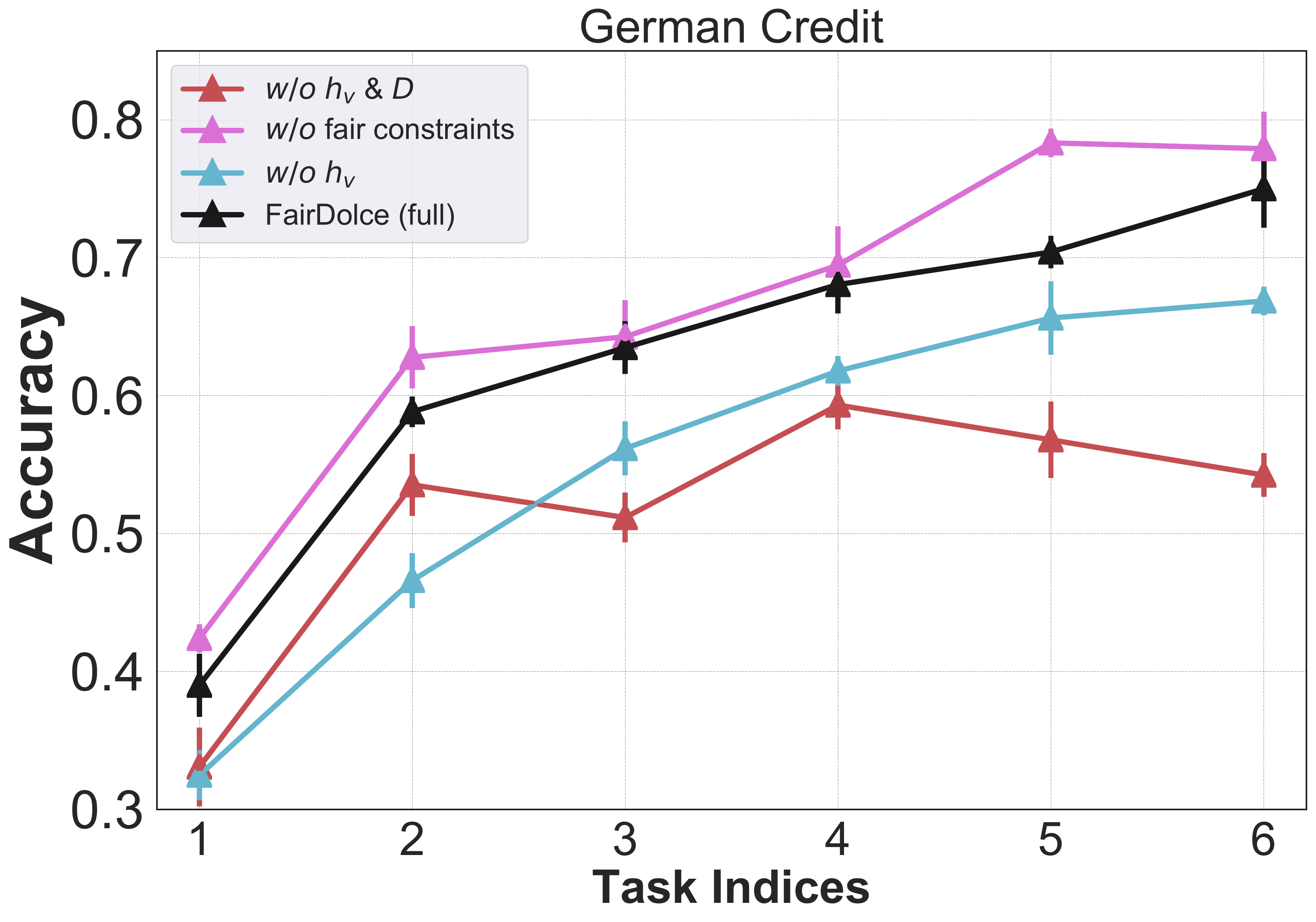}
    \end{subfigure}
    \caption{Ablation studies on the German Credit dataset.}
    \label{fig:abs-german}
\end{figure*}